\documentclass[preprint,11pt]{elsarticle}
\usepackage{amsthm}
\usepackage{amssymb}
\usepackage{graphicx}
\usepackage{picinpar}
\usepackage{subfigure}
\usepackage{booktabs}
\usepackage{float}
\usepackage{url}
\usepackage{mdwlist}
\usepackage{epstopdf}
\usepackage{hyperref}
\usepackage{amsfonts}
\usepackage{mathrsfs}
\usepackage{amsmath}
\usepackage{dsfont}
\usepackage{multirow}
\usepackage{color}
\usepackage{graphicx}
\usepackage{subfigure}
\usepackage{subfloat}
\usepackage{booktabs, tabularx}
\usepackage{bm}
\usepackage{algorithmic}
\biboptions{sort&compress}
\usepackage[top=1in, bottom=1.1in, left=1.25in, right=1.25in]{geometry}

\graphicspath{{fig/}}

\numberwithin{equation}{section}
\newtheorem{theorem}{\hspace*{2.0em}Theorem}[section]
\newtheorem{definition}{\hspace*{2.0em}Definition}[section]
\newtheorem{lemma}{\hspace*{2.0em}Lemma}[section]
\newtheorem{corollary}{\hspace*{2.0em}Corollary}[section]

\newtheorem{remark}{\hspace*{2.0em}Remark}[section]

\usepackage{lineno,hyperref}

\begin{document}
\begin{frontmatter}
\title{\Large {Feature Qualification by Deep Nets: A Constructive Approach}}

\author[1,2]{Feilong Cao}
 \ead{caofeilong88@zjnu.edu.cn}
 \author[3]{Shao-Bo Lin } \cortext[*]{Corresponding author: S. B. Lin}
  \ead{sblin1983@gmail.com}

\address[1]{School of Mathematics, Zhejiang Normal University, Jinhua 321014, China}

\address[2]             
    {Institute of Mathematics and Cross-disciplinary Science, Zhejiang Normal University,  Hangzhou 310012, China}


     \address[3]{Center for Intelligent Decision-Making and Machine Learning, School of Management, Xi'an Jiaotong University, Xi'an 710049, China}


\begin{abstract}
The great success of deep learning has stimulated avid research activities in verifying  the power of depth in theory, a common consensus of which is that deep net are  versatile in approximating and learning numerous functions. Such a versatility certainly enhances the understanding of the power of depth, but makes it difficult to judge which data features are crucial in a specific learning task. This paper proposes a constructive approach to equip deep nets for the feature qualification purpose. Using the product-gate nature and localized approximation property of deep nets  with sigmoid activation (deep sigmoid nets), we succeed in constructing a linear deep net operator that  possesses optimal approximation performance in approximating smooth and radial functions. Furthermore,   we   provide theoretical evidences that the constructed deep net operator is capable of qualifying multiple features such as the smoothness and radialness of the target functions.

\end{abstract}
\end{frontmatter}

\section{Introduction}

The great success of deep learning in practice \cite{goodfellow2016deep} significantly  stimulates the theoretical understanding of deep neural networks (deep nets), including to pursue the  excellent  performances in approximating numerous functions \cite{elbrachter2021deep}, to derive the capacity of deep nets under different metrics \cite{guo2019realizing}, to show their tractability in learning high-dimensional data \cite{schmidt2020nonparametric},  to analyze the global landscapes for some convex losses concerning deep nets  \cite{sun2020global}, and to settle certain convergence issues for gradient-based algorithms \cite{allen2019convergence}. These encouraging developments not only build a meaningful springboard to discover the running mechanism of deep learning, but also provide promising guidance in designing new deep learning schemes. However, several issues such as the inconsistency between generalization and optimization on the size of networks \cite{lin2021generalization} and the structure selection for given  learning tasks \cite{han2020depth} still remain open. More importantly,  there is a crucial dilemma between the versatility and feature qualification for deep nets in the sense that   the versatility  prohibits deep learning to figure out which data features are extracted in the approximation and learning processes.

The aim of this paper is to settle the dilemma between the versatility and feature qualification by   constructing  deep net operators. Our study is motivated by three interesting observations. At first,   inverse theorems that qualifies the smoothness via the approximation rates have been successfully derived for polynomials \cite{devore1993constructive}, radial basis functions \cite{hangelbroek2018inverse}, and shallow nets \cite{qian2022neural}. It is highly desired to establish similar inverse theorems for deep nets. Then, the versatility of  whole deep net sets makes it impossible to qualify a single data feature by approximation rates. For example, if the target functions are smooth and spatially sparse \cite{lin2018generalization,chui2020realization,Liuapproximating2023}, then the approximation rates are essentially products of  two quantities measuring the sparseness and smoothness. Without additional a-priori information, it is  extremely difficult to factorize these products to reflect the detailed smoothness and sparsity. A preferable way to conquer this obstacle is to narrow the range of   deep nets to compromise  their versatility. Finally, constructing rather than training deep nets succeeds in avoiding the inconsistency between optimization and generalization in learning
smooth  functions \cite{wang2024component}, spatially sparse functions \cite{liu2022construction} and piecewise smooth functions \cite{Liuapproximating2023}. The construction approach presented in \cite{qian2022neural,qian2022rates02} is also capable of yielding an inverse theorem for shallow nets in terms that the smoothness of the target functions can be qualified by the approximation rates. Based on these
observations, we devote to constructing deep nets that are capable of breaking through  bottlenecks of shallow nets, and qualifying the data features   by approximation rates, simultaneously.

As it has been verified in \cite{konovalov2008approximation,konovalov2009approximation,chui2019deep} that both   classical algebraic polynomials and shallow nets are incapable of extracting the radial property of the target functions, we aim to construct deep nets to approximate smooth and radial functions to embody the power of depth. We adopt the sigmoid function  to be the nonlinear activation since its differentiability   permits to establish the so-called Bernstein-type inequality  that is crucial   in feature qualification \cite{devore1993constructive,hangelbroek2018inverse,qian2022neural}. Utilizing  the square-gate \cite{chui2019deep} and localized approximation properties of deep sigmoid nets  \cite{chui1994neural}, we succeed in constructing a linear deep-neural-net-operator (DNO) for the approximation and feature qualification purposes. From the approximation perspective, we derive almost optimal  rates for the constructed DNO in approximating smooth and radial functions  which is beyond the capability of shallow nets with arbitrary activation functions, demonstrating the power of depth in our construction. Furthermore, the derived approximation rates are essentially the same as the those for training-based deep sigmoid nets \cite{chui2019deep}, showing that the construction approach does not degenerate the approximation performance of deep nets. From the feature qualification perspective, we  prove that the construction DNO is an optimal conditional qualification operator  with respect to the smoothness and radialness. In particular, we verify that if  the smoothness of target the target functions is    known, then 
the radialness information can be sufficiently reflected by the approximation rates of DNO, and vice-verse.

The rest of the paper is organized as follows. In the next section, we introduce some basic definitions concerning deep sigmoid nets. In Section \ref{sec3}, we construct DNO and study its approximation performance. In Section \ref{sec4}, the feature qualification property of the constructed DNO is analyzed.  In Section \ref{sec5}, we present the stepping stone of our theoretical analysis and the proofs of main results are given in the last section.

\section{Deep Nets: Approximation and Feature Qualification}\label{sec2}
A deep    net  with depth $L\in\mathbb N$ and width vector $(d_1,\dots,d_L)^T$ for $d_\ell\in\mathbb N$ is mathematically  defined  by
\begin{equation}\label{deep-net}
     \mathcal N_{d_1,\dots,d_L}(x)
     = a^L\cdot  \sigma\circ \mathcal J_{L,W^L,b^L} \circ \sigma  \circ \dots \circ \sigma\circ\mathcal J_{1,W^1,b^1}(x),\qquad x\in\mathbb R^d,
\end{equation}
where $f_1\circ f_2(x)=f_1(f_2(x))$,
${ a}^L \in\mathbb R^{d_L}$, $\sigma:\mathbb R\rightarrow\mathbb R$ is the activation function with $\sigma(\vec{u})$ acting on vectors componentwise of the vector $\vec{u}$, and
$\mathcal J_{\ell,W^\ell, b^\ell}:\mathbb R^{d_{\ell-1}}\rightarrow \mathbb R^{d_\ell} $ is the
  affine operator given by
$
    \mathcal J_{\ell,W^\ell, b^\ell}(x):=W^\ell  x+b^\ell
$
with $d_\ell\times d_{\ell-1}$ weight matrix  $W^\ell$ and $d_\ell$ dimensional bias vector  $b^\ell$.
Deep nets defined  by \eqref{deep-net} possess at most
\begin{equation}\label{number-parameters}
   \tilde{n}_L:=d_L+\sum_{k=1}^L (d_{k-1}d_k+d_{k})
\end{equation}
tunable parameters with $d_0=d$.
Denote by $\mathcal H_{d_1,\dots,d_L}$ the set of all deep nets formed as \eqref{deep-net}. If $L=1$, $\mathcal N_{d_1}$ defined by \eqref{deep-net} is the classical shallow net.   The structure of the deep net is determined by the weight matrices $W^1,\dots,W^L$. In particular, sparse matrix $W^\ell$, full matrix $W^\ell$ and Toeplitz matrix correspond to deep sparse connected  networks (DSCNs) \cite{petersen2018optimal}, deep fully connected networks (DFCNs) \cite{yarotsky2017error} and deep convolutional neural networks (DCNNs) \cite{zhou2020universality}, respectively.

The study of advantages of deep nets   can date back to the 1990s, when \cite{chui1994neural}  proved that deep nets can provide
 localized approximation   but shallow nets fail. After then,  limitations of shallow nets have been widely discovered, including their bottlenecks in approximating smooth functions \cite{chui1996limitations},  extracting the radialness   \cite{konovalov2008approximation}, capturing the
sparseness in the frequency domain \cite{lin2017does} and spatial domain \cite{lin2018generalization}, etc.. We refer the readers to  two fruitful review papers \cite{pinkus1999approximation,elbrachter2021deep}  for more details.   Under this circumstance, the power of depth has been verified in terms of avoiding the above bottlenecks in avoiding the saturation phenomenon of shallow nets \cite{yarotsky2017error}, capturing the radialness \cite{chui2019deep}, grasping the group-features of inputs \cite{han2020depth}, approximating non-smooth functions \cite{Ingo2021Approximation} and extracting the sparseness \cite{lin2017does,chui2020realization}. Table \ref{Tab:power} presents  several examples to show  advantages of deep nets.

\begin{table}[!h]
\vspace{-0.1in}
\caption{Power of depth (approximation within accuracy $ \varepsilon$)}\label{Tab:power}
\vspace{-0.1in}
\begin{center}
\begin{tabular}{|l|l|l|l|l|l|l|}
\hline Ref. & Feature   & Activation & Width &Depth&Structure\\
\hline \cite{chui1994neural} & localized approximation&   sigmoid &  $\mathcal O(d)$ & 2 &DFCN\\
\hline \cite{chui2020realization} & localized approximation&   ReLU &  $\mathcal O(d)$ & 2 &DFCN\\
\hline \cite{chui2019deep} & radial & sigmoid   & $\mathcal O(d)$  &  1  & DSCN\\
\hline
\cite{han2020depth} & group-feature & ReLU  &$\mathcal O(\varepsilon^{-1})$ & $\mathcal O(d)$ &DSCN\\
\hline \cite{nakada2020adaptive} & manifold & ReLU   & $\mathcal O(\varepsilon^{-1})$ &$\mathcal O(d)$ &DSCN\\
\hline \cite{lin2017does}& sparseness ($k$) &  sigmoid & $\mathcal O(k)$& $\mathcal O(\log k)$ &DFCN\\
\hline \cite{han2023learning} & translation-invariance & ReLU   & $\mathcal O(d)$ & $\mathcal O(\varepsilon^{-1})$ &DCNN\\
\hline \cite{mhaskar2016deep} & composite ($\ell$) & sigmoid &$\mathcal O(d)$ & $\mathcal O(\ell)$ & DSCN\\
\hline
\end{tabular}
\end{center}
\end{table}

It can be concluded from Table \ref{Tab:power} that
different features require totally different structures, making the selection of networks structure to be quite difficult. Moreover, the interactions among different deep nets  makes it unclear whether a simple tandem or parallel manner \cite{petersen2018optimal} of the obtained networks in Table \ref{Tab:power} is feasible to generate a novel deep net to extract more than one features. Even the answer is feasible, it is almost impossible to theoretically  judge which features play crucial roles while others are not so important for a specific learning tasks. This phenomenon, caused by the versatility of deep nets presented in Table \ref{Tab:power},  shows that it is difficult to judge the features of the data via the given training error as long as there are more than one data features are involved, and is
one of the most important obstacles to understand deep learning. Therefore, there is a dilemma between the versatility and feature qualification for deep nets in the sense that the former requires deep nets to be able to tackle various data features, but the latter  restricts performances of deep nets in extracting multiple features
 so that the approximation rates are sufficient to qualify the   features.






Without loss of generality, we focus our analysis on the unit ball $\mathbb B^d\subset\mathbb R^d$. Let $L^p(\mathbb B^d)$ with $1\leq p\leq\infty$ be the classical space of $p$-th Lebesgue integrable function space. For
$U,V\subseteq L^p(\mathbb B^d)$,  denote
by
$$
     \mbox{dist}(U,V, L^p(\mathbb B^d)):=\sup_{f\in U} \mbox{dist}(f,V, L^p(\mathbb B^d))
     :=\sup_{f\in U} \inf_{{ g}\in V}\|f-{  g}\|_{L^p(\mathbb
       B^d)}
$$
the deviations of $U$ from $V$ in  $L_p(\mathbb B^d)$. If $p=\infty$, we write $\mbox{dist}(U,V)=\mbox{dist}(U,V,L^\infty(\mathbb B^d)$ for the sake of brevity.  We then present  definitions of (almost) optimal deep-net-operator (DNO) to describe the approximation performance of a linear operator $G_{d_1,\dots,d_L}: L^\infty(\mathbb B^d)\rightarrow \mathcal H_{d_1,\dots,d_L}$.

\begin{definition}\label{Def:DNO}
Let $L,d_1,\dots,d_L\in\mathbb N$, $U\subseteq L^\infty(\mathbb B^d)$ and $G_{d_1,\dots,d_L}: L^\infty(\mathbb B^d)\rightarrow \mathcal H_{d_1,\dots,d_L}$ be a linear operator.  Denote further $\mathcal G_{d_1,\dots,d_L,U}:=\{G_{d_1,\dots,d_L}f:f\in U\}$. If
\begin{equation}\label{cond-DNO}
    \mbox{dist}(U,\mathcal G_{d_1,\dots,d_L,U})\asymp \mbox{dist}(U,\mathcal H_{d_1,\dots,d_L}),
\end{equation}
we then say that $  G_{d_1,\dots,d_L}$ is an optimal DNO for $U$ in $\mathcal H_{d_1,\dots,d_L}$, where $a\asymp b$ for $a,b>0$ means that there is a constant $C>0$ independent of $d_1,\dots,d_L$ such that $C^{-1}a\leq b\leq Ca$.
  If there is a $v\geq0$ such that
$$
   \mbox{dist}(U,\mathcal H_{d_1,\dots,d_L})\leq \mbox{dist}(U,\mathcal G_{d_1,\dots,d_L,U})\leq C'\mbox{dist}(U,\mathcal H_{d_1,\dots,d_L})\log^v{\tilde{n}_L},
$$
we then say that $  G_{d_1,\dots,d_L}$ is an $v$-almost optimal DNO for $U$ in $\mathcal H_{d_1,\dots,d_L}$.
\end{definition}

The definition of optimal DNO demonstrates the approximation performance of the linear operator $G_{d_1,\dots,d_L}$ and thus provides a baseline of   construction. Actually, how to develop an optimal DNO to embody the advantage of deep nets is challenging and interesting, since such a DNO avoids training the networks and therefore circumvents the well known inconsistency between optimization and generalization \cite{lin2021generalization}. We then provide some  definitions concerning conditional qualification operators to measure the feature qualification performance of the linear operator.

\begin{definition}\label{Def:condition-inverse}
Let $\Theta_{1,r}$ and $\Theta_{2,s}$ be two sets of functions with indices  $r>0$ and $s>0$. If $G_{d_1,\dots,d_L}:L^\infty(\mathbb B^d)\rightarrow\mathbb \mathcal H_{d_1,\dots,d_L}$ is  a linear operator  and
\begin{equation}\label{qua-cond-1}
    \|f-G_{d_1,\dots,d_L}(f)\|_{L^\infty(\mathbb B^d)}\leq u_2(\tilde{n}_L,s,r),\qquad\forall f\in\Theta_{1,r}
\end{equation}
implies $f\in \Theta_{1,r}\cap\Theta_{2,s}$,  then  $G_{d_1,\dots,d_L}$ is said to be a qualification operator for $\Theta_{2,s}$ with qualification rate $u_2(\tilde{n}_L,s,r)$ conditioned on $\Theta_{1,r}$, which is denoted by  $\Theta_s \overset{u_2(\tilde{n}_L,s,r)}{\propto}  G_{d_1,\dots,d_L}|_{\Theta_{1,r}} $. If
\begin{equation}\label{qua-cond-2}
    \|f-G_{d_1,\dots,d_L}(f)\|_{L^\infty(\mathbb B^d)}\leq  u_1(\tilde{n}_L,r,s),\qquad\forall f\in\Theta_{2,s}
\end{equation}
implies $f\in \Theta_{1,r}\cap\Theta_{2,s}$, then   $\Theta_{1,r} \overset{u_1(\tilde{n}_L,r,s)}{\propto} G_{d_1,\dots,d_L}|_{\Theta_{2,s}} $, showing that $G_{d_1,\dots,d_L}$ is  a qualification operator for $\Theta_{1,r}$ with qualification rate $u_{1}(\tilde{n}_L,r,s)$ conditioned on $\Theta_{2,s}$.
\end{definition}

If $\Theta_{1,r}=L^\infty(\mathbb B^d)$, then    the condition $f\in \Theta_{1,r}$ always holds. Under this circumstance, the condition $  \Theta_{2,s} \overset{u_2(\tilde{n}_L,s,r)}{\propto} G_{d_1,\dots,d_L}|_{\Theta_{1,r}}$ is removable and we denote by $\Theta_{2,s} \overset{u_2(\tilde{n}_L,s,r)}{\propto} G_{d_1,\dots,d_L}$ the fact that $G_{d_1,\dots,d_L}$ is a qualification operator for $\Theta_{2,s}$. To be detailed, $\Theta_{2,s} \overset{u_2(\tilde{n}_L,s,r)}{\propto} G_{d_1,\dots,d_L}$ refers to use $G_{d_1,\dots,d_L}$ to qualify a single feature, which has been studied in  \cite{qian2022neural}  when $ \Theta_{2,s} $ is the  Sobolev space with smoothness index $s$ \cite{adams2003sobolev} and $L=1$.
Introducing  conditional qualification operators provides a reasonable way to   understand the running mechanism of deep learning due to the excellent performance of deep nets in extracting multiple features. Actually,  some a-priori information is achievable in some  application regions, 
making part of the data features   clear to the users and the proposed conditional qualification powerful. For instance, images are always piecewise smooth \cite{imaizumi2019deep};  military signals are frequently referred as sparse in the spatial domain \cite{liu2022construction};  and earthquake magnitude predictions are concerned with radial functions \cite{han2020depth}. In the above definition, we only present the conditional qualification classes for two features for the sake of convenience. Similar definitions for multiple features can be derived in a similar way.
Based on Definition \ref{Def:DNO} and Definition \ref{Def:condition-inverse}, we present the following definition of essential DNO.

\begin{definition}\label{Def:essential-DNO}
Let $\Theta_{1,r}$ and $\Theta_{2,s}$ be two sets of functions with indices  $r>0$ and $s>0$.
If  $G_{d_1,\dots,d_L}: L^\infty(\mathbb B^d)\rightarrow \mathcal H_{d_1,\dots,d_L}$ is an optimal  DNO for $\Theta_{1,r}\cap \Theta_{2,s}$ and
\begin{equation}\label{two-side-qualification}
   \Theta_s \overset{u(\tilde{n}_L,s,r)}{\propto}  G_n|_{\Theta_{1,r}},\qquad  \Theta_{1,r} \overset{u(\tilde{n}_L,r,s)}{\propto} G_n|_{\Theta_{2,s}}
\end{equation}
 holds  for $
    u(\tilde{n}_L,s,r) =\mbox{dist}(\Theta_{1,r}\cap \Theta_{2,s},\mathcal H_{d_1,\dots,d_L}),
$
  $G_n$ is said to be an essential DNO  with respect to $\Theta_{1,r}\cap\Theta_{2,s}$. If  $G_{d_1,\dots,d_L}$ is a $v$-almost optimal  DNO and
\eqref{two-side-qualification} holds  for $
    u(\tilde{n}_L,s,r) =\mbox{dist}(\Theta_{1,r}\cap \Theta_{2,s},\mathcal H_{d_1,\dots,d_L})\log^v\tilde{n}_L,
$
$G_n$ is said to be a $v$-almost essential DNO  with respect to $\Theta_{1,r}\cap\Theta_{2,s}$.
\end{definition}

Introducing essential DNO is important since the conditional qualification operator given in Definition \ref{Def:condition-inverse} lacks in building the relationship between  conditional qualification rates and the data features. To be detailed, it is difficult to judge whether the data features are fully extracted by using \eqref{qua-cond-1} or \eqref{qua-cond-2}. Taking $\Theta_{1,r}=L^\infty(\mathbb B^d)$ and $ \Theta_{2,s} $ as the  Sobolev space with smoothness index $s$ and $d=1$ for example, any qualification rates of order  $n^{-r}$ with $r\geq s$ implies $f\in\Theta_{2,s}$, showing a gap for the  conditional qualification operator to describe the performance of feature qualification of the operator $G_{d_1,\dots,d_L}$. The essential DNO fills the above
gap in the sense that   an essential operator is capable of sufficiently qualify the data feature without any compromise.

\section{Construction of Optimal Deep-Net-Operator and Approximation Error Analysis}\label{sec3}

According to Definition \ref{Def:DNO}, it is meaningless to construct deep-net-operator without concerning  any data features. In this paper, we are interested in  two data features on  the smoothness and radialness of  target functions, since both the classical polynomials \cite{konovalov2008approximation} and shallow nets \cite{konovalov2009approximation} are incapable of deriving optimal approximation rates for such functions.  The following  definition depicts $\tau$-radial functions.
\begin{definition}\label{Def:tau-radial}
Given   $\tau>0$ and $f\in L^\infty(\mathbb B^d)$,
if there exists a $g_f:[0,1]\rightarrow \mathbb R$ such that
\begin{equation}\label{near-radial}
    \max_{x\in\mathbb B^d}|f(x)-g_f(\|x\|_2^2)|\leq \tau,
\end{equation}
  then  $f$ is said to be $\tau$-radial. Denote by $\mathcal R_\tau$ the set of all $\tau$-radial functions.
\end{definition}
It is obvious that $0$-radial functions correspond to the classical radial functions \cite{konovalov2008approximation,konovalov2009approximation,chui2019deep,han2020depth}. Radial functions play crucial roles in earthquake early warning \cite{saad2020deep}, seismic prediction \cite{zhang2018deep} and oil exploration \cite{temirchev2020deep}, and reflect the rotation-invariance of data \cite{chui2019deep}.
Besides the radialness, we introduce the near-smooth functions as follows.

\begin{definition}\label{Def:Lip+}
Let   $ 0<\alpha\leq 1$, $c>0$ and $\mathbb A\in\mathbb R^{d'}$ for some $d'\in\mathbb N$. If for any $x,x'\in\mathbb A$, there holds
\begin{equation}\label{lipt}
    |\tilde{f}(x)-\tilde{f}(x')|\leq c\|x-x'\|_2^\alpha, \qquad 0<\alpha\leq 1,
\end{equation}
we   say that $\tilde{f}$ is $\alpha$ Lipchitz continuous with coefficient $c$, where  $\|\cdot\|_2$ denotes the Euclidean norm for $\mathbb R^{d'}$. Denote by $Lip^{\alpha,c}(\mathbb A)$ the set of all such functions.
Given $\nu>0$, if there is a $\tilde{f}\in Lip^{\alpha,c}(\mathbb A)$ satisfying
\begin{equation}\label{near-lip}
    \max_{x\in\mathbb A}|f(x)-\tilde{f}(x)|\leq \nu,
\end{equation}
then  $f$ is called as $\alpha$ Lipchitz continuous with coefficient $c$ and tolerance $\nu$. We denote  by $Lip^{\alpha,c,\nu}(\mathbb A)$ the set of all such $f$.
\end{definition}

If $\mathbb A=\mathbb B^d$, we denote by $Lip^{\alpha,c}$ for $Lip^{\alpha,c}(\mathbb  B^d)$ and $Lip^{\alpha,c,\nu} $ for $Lip^{\alpha,c,\nu}(\mathbb B^d)$, if no confusion is made.
The smoothness described as in Definition \ref{Def:Lip+} shows that small perturbation of inputs does not change the outputs much, which abounds in numerous applications \cite{bengio2013representation} and has been adopted in vast literature \cite{yarotsky2017error,petersen2018optimal,lin2018generalization,shaham2018provable,schmidt2020nonparametric,zhou2020universality} to demonstrate the feasibility and efficiency  of deep learning.
Introducing the tolerance for the classical Lipchitz continuous functions is to better quantify the smoothness and establish the relation of smoothness between $f$ and $g_f$ in Definition \ref{Def:tau-radial}, just as the following lemma purports to show.

\begin{lemma}\label{Lemma:lip}
Let $0<\alpha\leq 1$ and $c,\nu,\tau>0$. If $f\in\mathcal R_\tau$ and $g_f$ is the corresponding univariate function given in \eqref{near-radial}, then $g_f\in Lip^{\alpha,c,\nu}([0,1])$ implies $f\in  Lip^{\alpha,2^\alpha d^{\alpha/2}c,2\tau+\nu}(\mathbb B^d)$ and $f\in Lip^{\alpha,c,\nu}(\mathbb B^d)$ implies $g_f\in Lip^{\alpha,c,2\tau+\nu}([0,1])$.
\end{lemma}

The following lemma  deduced from \cite{konovalov2009approximation,chui2019deep}
quantifies the limitation of shallow nets in qualifying the radialness.

\begin{lemma}\label{Lemma:lower bound for deep nets 1}
Let $d\geq 2$, $n\in \mathbb N,0<\alpha\leq1,c>0$, and $\nu,\tau\geq0$. For any continuous activation   $\sigma$, there holds
 \begin{equation}\label{limitation for shallow}
   \mbox{dist}(Lip^{\alpha,c,\nu}\cap \mathcal R_\tau,\mathcal H_{d_1},L^\infty(\mathbb B^d))
   \geq
    C^*_1{d_1}^{-\frac{r}{d-1}}
\end{equation}
where  $C_1^*$ is a  constant  independent of $n$.
\end{lemma}

According to  Definition \ref{Def:tau-radial}, it is easy to get that  $\tau$-radial functions can be well approximated by   some univariate functions. But
Lemma \ref{Lemma:lower bound for deep nets 1} presents a dimension-dependent lower bound for shallow nets, demonstrating that shallow nets with arbitrary continuous activation cannot grasp such a univariate property and thus fail  to reflect the $\tau$-radialness.
Therefore, it is necessary to deepen the network to capture both smooth and radial features. Actually,
it can be found in \cite{chui2019deep,han2020depth} that such a bottleneck can be easily circumvented by using  deep sigmoid nets or deep ReLU nets. However, the deep nets in \cite{chui2019deep} and \cite{han2020depth} require to solve highly non-convex optimization problem, lacking provable convergent algorithms \cite[Chap.8]{goodfellow2016deep}.





We therefore focus on constructing an DNO in $\mathcal H_{d_1,\dots,d_L}$ to capture both smoothness and radialness of target functions. As there is not any optimization problem involved in  the construction process, we use the sigmoid activation function, i.e., $\sigma(t)=\frac1{1+e^{-t}}$, to take place of the widely used ReLU activation, mainly due to the following three considerations.  At first, compared with   deep ReLU nets, deep sigmoid nets dominate in deriving similar approximation rates  with much less depth. For example, to approximate the well known square function $t^2$ to an accuracy $\varepsilon$, a deep ReLU net with depth $\log \varepsilon^{-1}$ and width $\log \varepsilon^{-1}$ is constructed in \cite{yarotsky2017error}, while for sigmoid net, it can be found in \cite{chui2019deep} that one hidden layer and three neurons are sufficient. Then, as discussed  in \cite{zeng2021admm}, it is easy to represent the ReLU function by a deep sigmoid net with two hidden layers and $\mathcal O(d)$ neurons but not vice-verse.   Finally, due to the differentiability of the sigmoid function, it is possible to derive a Bernstein-type inequality in approximation theory \cite[Chap.7]{devore1993constructive} to bound the $L^\infty$-norm of the derivative of some deep (or shallow) sigmoid nets by the $L^\infty$-norm of themselves \cite{qian2022neural}, which is crucial to the verify the qualification classes of smooth functions and is beyond the capability of deep ReLU nets.

Our construction is based on  several important properties of deep (or shallow) sigmoid nets.
At first,  we present a procedure to construct a shallow sigmoid net with $k+1$ neurons to approximate a univariate polynomial $p_k(t)=\sum_{i=0}^ku_{i}t^i$ with $u_k\neq0$  in an arbitrary accuracy $\varepsilon.$
Beginning with  $N_{k+1,p_k}(t)=0,$ for $j=k,\ldots,1$, we define $u_i^{(k)}=u_i$, and iteratively $
   \mu_j=\min\left\{1,\frac{\varepsilon|\sigma^{(j)}(0)|(j+1)}{|u_i^{(j)}|
       \max_{-1\leq t\leq
          1}|\sigma^{(j+1)}(t)|}\right\},
$
$
   u_i^{(j-1)}=u_i^{(j)}-\frac{u_{i}^{(j)}j!\sigma^{(i)}(0)}{\sigma^{(j)}(0)\mu_j^{j-i}i!}
$
and
\begin{equation}\label{iterative-polynomial}
    N_{j,p_k,\varepsilon}(t)=N_{j+1,p_k,\varepsilon}(t)+ {u}^{(j)}_{j}\frac{j!}{\mu_{j}^{j}\sigma^{(j)}(0)}\sigma(\mu_{j}t).
\end{equation}
It can be easily deduced from the above procedure that $N_{0,p_k,\varepsilon}(t)$ is a shallow net with $k+1$ neurons and its weights are bounded by $
C_0'\left(1+\sum_{i=0}^{k}|u_{i}|\right)^{(k+1)!}\varepsilon^{-(k+1)!}$, where $C_0'\geq1$ is an absolute constant. In particular,
it was derived in \cite[Proposition 3.3]{chui2019deep} the following lemma.

\begin{lemma}\label{Lemma:polynomial for nn}
Let $k\in\{0,\dots,s_0\}$ and  $p_k(t)=\sum_{i=0}^ku_{i}t^i$. For an arbitrary
$\varepsilon\in(0,1)$,  there holds
\begin{equation}\label{error estimate pk h}
       |p_k(t)-N_{0,p_k,\varepsilon}(t)|\leq  \varepsilon, \qquad
        \forall\ t\in [-1,1].
\end{equation}
\end{lemma}

Based on the above lemma,
for any $x=(x^{(1)},\dots,x^{(d)})^T$, define
\begin{equation}\label{l2norm-app}
    N_{0,t^2,\varepsilon}(x)=\sum_{j=1}^dN_{0,t^2,\varepsilon}(x^{(j)}).
\end{equation}
 It is easy to check that  $N_{0,t^2,\varepsilon}(x)$ is a shallow net with width $3d$  satisfying
\begin{equation}\label{l2norm-app-rate}
     |N_{0,t^2,\varepsilon}(x)-\|x\|_2^2|\leq d\varepsilon.
\end{equation}

Due to Lemma \ref{Lemma:polynomial for nn} again, we can derive the following lemma concerning a product-gate property of the shallow sigmoid net \cite{chui2019deep}.

\begin{lemma}\label{Lemma:product-gate}
For $\varepsilon\in(0,1)$, there exists a
shallow sigmoid net $
       \tilde{\times}_2: [-1,1]^2\rightarrow\mathbb R
$ with 9 hidden neurons whose weights are bounded by $ C_0\varepsilon^{-6} $  such that for any $t,t'\in[-1,1]$,
\begin{equation}\label{product-gate}
       |tt'-\tilde{\times}_2(t,t')|\leq
       \varepsilon,
\end{equation}
where $C_0$ is an absolute constant.
\end{lemma}

Define
\begin{equation}\label{bell-fun-def}
\varphi(t):=\frac{\sigma(t+1)-\sigma(t-1)}2,
\quad
t\in \mathbb R.
\end{equation}
It was derived in
 \cite{Chen2009The}  the following lemma.

\begin{lemma}\label{Lemma:property-mean}
 The function $\varphi$ defined by \eqref{bell-fun-def} is  positive and bell-shaped symmetric and satisfies:\\
 $\bullet$ (I): $\int_{-\infty}^{+\infty}\varphi(t){\textrm d}t=1$;\\
 $\bullet $ (II): $\hat{\varphi}(i)=0$ for any $i\in \mathbb Z\ \{0\}$, where $\hat{\varphi}(i)$ denotes $i$-th Fourier transforms of $\varphi$; \\
 $\bullet$ (III): For any $t\in (-\infty,+\infty)$, $\sum^{+\infty}_{i=-\infty}\varphi(t-i)=1$.
\end{lemma}

Based on the above preliminaries, we are in a position to present
 our construction.
For any $f$ defined on $\mathbb B^d$ and $n\in\mathbb N$,
denote
\begin{equation}\label{Def.xi}
      \xi_k=\left\{\begin{array}{ll}
(\overbrace{0,\dots,0}^{d},-1)^T, &0\leq k\leq n-1;\\
\left(\overbrace{0,\dots,0}^{d-1},\frac{k-2n}{n}\right)^T, & n\leq k\leq 3n;\\
 (\overbrace{0,\dots,0}^{d-1},1)^T, & 3n+1\leq k\leq 4n.
\end{array}\right.
\end{equation}
Define
\begin{equation}\label{FNNopertors2}
    G_{n,f,\varepsilon}(x):=G_{n,\varepsilon}f(x)=\sum_{k=0}^{4n}f(\xi_k)\varphi(nN_{0,t^2,\varepsilon}(x)-k+2n)
\end{equation}
with $N_{0,t^2,\varepsilon}$ being given in \eqref{l2norm-app} for some $\varepsilon>0$. It is easy to see that $G_{n,\varepsilon}: L^\infty(\mathbb B^d)\rightarrow\mathcal H_{3,4n}$ is a linear operator.
Define further
\begin{equation}\label{class-of-construction}
    \mathcal G_{n,\varepsilon}:=\{G_{n,f,\varepsilon}:f\in L^\infty(\mathbb B^d)\}.
\end{equation}
Then, $\mathcal G_{n,\varepsilon}\subseteq \mathcal H_{3,4n}$ can be regarded as linear sub-space with dimension $4n+1$, provided $\varepsilon$  is specified in the construction.

To quantify the approximation performance of the constructed deep sigmoid net, we introduce the  modulus of continuity \cite{Lorentz1966Approximation,
Xie1998Approximation,
Zygmund2002Trigonometric} that is an important tool to characterize the smoothness.
For $g\in L^\infty({[a,b]})$, its modulus of continuity
is defined by
\begin{equation}\label{contin-modu-def}
\omega(g,h):=\sup_{0<s\leq h}\max_{x, x+s\in[a,b]}|g(x)-g(x+s)|,\qquad h>0.
\end{equation}
Our first main result, as shown in the following theorem, presents an approximate rates of $G_{n,f,\varepsilon}$.



\begin{theorem}\label{directTheorem}
Let $ \tau,\varepsilon>0$. For any $f\in \mathcal R_\tau$ with corresponding univariate function $g_f$, we have
\begin{equation}\label{Error3}
\|f-G_{n,f,\varepsilon}\|_{L^\infty(\mathbb B^d)}\leq
 2\tau+2d\omega(g_f,\varepsilon)+\frac{10\rm e^2-3}{ \rm e^2}\omega\left(g_f,\frac 1{{n}}\right)+3 {\rm e}^{-{n}} (\|f\|_{L^\infty(\mathbb B^d)}+\tau).
\end{equation}
\end{theorem}

Theorem \ref{directTheorem} presents an upper bound of the approximation performance for the constructed deep sigmoid net $G_{n,f,\varepsilon}$ when $f$ is $\tau$-radial.  A direct corollary is as follows.

\begin{corollary}\label{Corollary:approximation-direct}
If $f\in Lip^{\alpha,c,\nu}\cap\mathcal R_\tau$ with $0<\alpha\leq 1$, $c>0$, $\nu\leq n^{-\alpha}$ and $\tau\leq n^{-\alpha}$,  and $\varepsilon$ is specified to be smaller than $\frac1{n}$, then
\begin{equation}\label{App.-rate}
    \|f-G_{n,f,\varepsilon}\|_{L^\infty(\mathbb B^d)}\leq C_1n^{-\alpha},
\end{equation}
where $C_1$ is a constant independent of $n,\varepsilon,\nu,\tau$.
\end{corollary}

Comparing \eqref{App.-rate} with Lemma \ref{Lemma:lower bound for deep nets 1}, we obtain that the constructed deep sigmoid nets conquer the bottleneck  of shallow nets in approximating radial and smooth functions. Actually, the approximation rates derived for $G_{n,f,\varepsilon}$ is independent of the dimension $d$ and thus succeed in  embodying the radialness of $f$. Since the weights of $N_{0,t^2,\varepsilon}$
are bounded by $
C_0'\left(1+\sum_{i=0}^{k}|u_{i}|\right)^{(k+1)!}\varepsilon^{-(k+1)!}$ that is quite large when
 $\varepsilon$ is extremely small, the constructed deep sigmoid net    is totally different from  deep nets derived by solving some optimization problems, as the weights for the latter are usually small \cite{zeng2021admm}.
Recalling  \cite[Theorem 2]{chui2019deep}, we have
\begin{equation}\label{Almost-optima-rate}
     \begin{split}
         C_2 (n\log_2 { n})^{-\alpha}
         &\leq \mbox{dist}(Lip^{\alpha,c,\nu}\cap \mathcal R_\tau,\mathcal H_{3,4n},L^\infty(\mathbb B^d))\\
   &\leq \mbox{dist}(Lip^{\alpha,c,\nu} \cap \mathcal R_\tau,\mathcal G_{n,\varepsilon},L^\infty(\mathbb B^d))
   \leq C_1n^{-\alpha}.
     \end{split}
\end{equation}
This together with  Definition \ref{Def:DNO} yields the following corollary.

\begin{corollary}\label{Corollary:optimal-DNO}
If  $0<\alpha\leq 1$, $c>0$, $\nu\leq n^{-\alpha}$, $\tau\leq n^{-\alpha}$,  and $\varepsilon$ is specified to be smaller than $\frac1{n}$, then
 $G_{n, \varepsilon}$ is an $\alpha$-almost optimal DNO for $Lip^{\alpha,c,\nu}\cap \mathcal R_\tau$ in $\mathcal H_{3,4n}$.
\end{corollary}

All  above results   show that the construction scheme does not essentially degrade the approximation performance of the classical optimization-based deep learning  schemes and depict that  the construction of deep sigmoid nets provides a novel direction to understand deep learning.

\section{Feature Qualification by the Constructed Deep-Net-Operator}\label{sec4}
The previous section presents the approximation rates for $  G_{n,\varepsilon}$ and shows that the constructed deep sigmoid net  outperforms shallow nets \cite{konovalov2009approximation} and the classical polynomials in approximating radial and smooth functions. Compared with training-based deep nets  in \cite{chui2019deep,han2020depth},  $  G_{n,\varepsilon}$, as exhibited in \eqref{FNNopertors2}, does not involve any  optimization problems and thus requires only $\mathcal O(n)$ computational complexity to yield almost optimal approximation rates \eqref{Almost-optima-rate}. All these show that $G_{n, \varepsilon}$ behaves as an almost optimal DNO. In this section,
 we aim to present justify $G_{n, \varepsilon}$  is also an almost essential DNO according to Definition \ref{Def:essential-DNO}. To this end, we need the following theoretical results to show that  $G_{n,\varepsilon}$  is
a conditional qualification operator  for smooth and radial functions and   the derived approximation rates in Theorem \ref{directTheorem} are the corresponding qualification rates. The following theorem shows that conditioned on the $\tau$-radialness, $G_{n,\varepsilon}$ is capable of qualifying the smoothness via approximation rates.

\begin{theorem}\label{Theorem:qualification-class}
Let $0<\alpha\leq1$, $c>0$. If $0<\varepsilon\leq Cn^{-1-\alpha}$ and $0\leq \tau\leq n^{-\alpha}$,
then
   $Lip^{\alpha,c,2\tau} \overset{c_1n^{-\alpha}}{\propto}  G_{n,\varepsilon}|_{\mathcal R_\tau} $ for some $c_1, C>0$ independent of $n,\varepsilon,\tau$.
\end{theorem}

Without the condition $f\in\mathcal R_\tau$, Theorem \ref{Theorem:qualification-class} is regarded as an inverse theorem in the approximation theory community. To be detailed, it can be deduced from \cite[Chap.7]{devore1993constructive}, \cite{hangelbroek2018inverse}, and \cite{qian2022neural}
that  linear operators based on polynomials, radial basis functions, and shallow nets  respectively   were developed to qualify the smoothness. However, if there are multiple features  involved in the approximation tasks, the aforementioned three approaches are difficult for feature qualification,  even when some data features are known. This was demonstrated in Lemma \ref{Lemma:lower bound for deep nets 1}. The main novelties of Theorem \ref{Theorem:qualification-class}, compared with \cite[Chap.7]{devore1993constructive}, \cite{hangelbroek2018inverse}, and \cite{qian2022neural}, are two folds. On one hand, classical feature qualification focuses on the single smoothness without considering any conditional restriction, while our result refers to more than one feature. On the other hand, since all of the  polynomials \cite{konovalov2008approximation}, radial basis functions \cite{lin2011essential} and shallow nets \cite{konovalov2009approximation} are incapable of reflecting the radialness, the qualification rates of order $\mathcal O(n^{-\alpha})$ is unachievable for these methods, which depicts the power of depth in our construction.  Since $\varepsilon$ can be specified to be extremely small in the construction, the restriction $0<\varepsilon\leq n^{-1-\alpha}$ is mild. Moreover, $0\leq\tau\leq n^{-\alpha}$ extends the condition from radial functions $(\tau=0)$ to $\tau$-radial functions, the price of which is that the smoothness under qualification is changed from $Lip^{\alpha,c}$ to  $Lip^{\alpha,c, 2\tau}$.
Similarly, we derive our second theorem concerning the feature qualification of $G_{n,\varepsilon}$.

\begin{theorem}\label{Theorem:qualification-class-2}
Let $0<\alpha\leq 1$, $c>0$. If $0<\varepsilon\leq n^{-1-\alpha}$ and $0\leq \nu\leq n^{-\alpha}$,
then
   ${\mathcal R_\tau}\overset{c_2n^{-\alpha}}{\propto}  G_{n,\varepsilon}|_{Lip^{\alpha,c,\nu} } $ for
   $\tau\geq C'n^{-\alpha}$, where $C',c_2$ are positive constants
 independent of $n,\varepsilon,\tau,\nu$.
\end{theorem}

Theorem \ref{Theorem:qualification-class-2} shows that $G_{n,\varepsilon}$ is able to qualify the radialness, which has not been considered  in the literature, since shallow nets and  polynomials  cannot reflect the radialness \cite{konovalov2008approximation,konovalov2009approximation,lin2011essential}. Furthermore, though the existing training-based deep nets dominate  in grasping the radialness  \cite{chui2019deep,han2020depth}, the versatility of deep nets makes it impossible qualify the radialness via  approximation rates. Our constructed DNO, $G_{n, \varepsilon}$, succeeds in embodying the advantage of deep nets in simultaneously capturing the smoothness and radialness  and can be utilized to qualify features by approximation rates.  It should be highlighted that the qualification rates of order $\mathcal O(n^{-\alpha})$ is almost optimal according to Corollary \ref{Corollary:optimal-DNO}. 
  It should be mentioned the condition   $\tau\geq Cn^{-\alpha}$ on the radialness is necessary, since the approximation rates formed as \eqref{App.-rate} cannot determine a $g_f$ satisfying \eqref{near-radial} with extremely  small $\tau$. Based on Theorem \ref{Theorem:qualification-class} and Theorem \ref{Theorem:qualification-class-2}, we can derive the following corollary to show that $G_{n, \varepsilon}$ is an almost essential DNO according to Definition \ref{Def:essential-DNO}.

\begin{corollary}\label{Corollary:Essen-DNO}
  Let $0<\alpha\leq1$, $c>0$. If $0<\varepsilon\leq n^{-1-\alpha}$, $0\leq \nu\leq n^{-\alpha}$, and   $\tau\asymp n^{-\alpha}$, then $G_{n,\varepsilon}$ is an $\alpha$-almost essential DNO  with respect to $Lip^{\alpha,c,\nu}\cap \mathcal R_\tau$.
\end{corollary}

Corollary \ref{Corollary:Essen-DNO} presents the excellent performance of the constructed DNO   in terms that it provides  provable feature qualification for multiple features without sacrificing the approximation performances of deep nets. We conclude this section with two remarks.

\begin{remark}\label{Remark:activation-function}
This paper only  focuses only on the sigmoid function $\sigma(t)=\frac1{1+e^{-t}}$ for the sake of brevity. Actually, it can be found in the proofs that similar results also hold for  non-decreasing
 and twice differentiable $\hat{\sigma}$ satisfying:\\
(i) general sigmoid property:
$$
\lim_{t\rightarrow+\infty}\hat{\sigma}(t)=R, \  \  \ \lim_{t\rightarrow-\infty}\hat{\sigma}(t)=r,
$$
for  $R, r\in \mathbb R$, and $R>r$;\\
(ii)  $\hat{\sigma}(t)-\frac{1}{2(R-r)}$ is an odd function;\\
(iii) $\hat{\sigma}(t)-R={\mathcal O}(|t|^{-1-\beta})$ as $t\rightarrow +\infty$ for some $\beta>0$;
$\hat{\sigma}(t)-r={\mathcal O}(|t|^{-1-\beta})$ as $t\rightarrow -\infty$ for some $\beta>0$.\\
Clearly, the mentioned  function
$\frac 1{1+\textrm e^{-t}}$, the hyperbolic tangent function
$
\frac{\textrm e ^t-\textrm e ^{-t}}{\textrm e ^t+\textrm e ^{-t}},
$
and arctan function
$\arctan(t)$  satisfy above restrictions on the activation. We can use the same construction and similar proofs to derive theoretical results as Theorem \ref{directTheorem}, Theorem \ref{Theorem:qualification-class} and Theorem \ref{Theorem:qualification-class-2} for the corresponding  DNO with these activation functions.
We leave the detailed construction and proofs for interested readers.
\end{remark}

\begin{remark}\label{Remark:others}
Different from \cite[Chap.7]{devore1993constructive} and \cite{hangelbroek2018inverse} where approximation rates and feature qualification are conducted for $\alpha$-smooth functions with $\alpha>1$, our derived results are only available for Liptchiz functions. For   approximation rates only, it is easy to deepen the networks further to capture the smoothness of high order, just as \cite{chui2019deep} and \cite{Liuapproximating2023} did for deep sigmoid nets. Indeed, we can use a localized Taylor polynomials approximation $T_n(f)$ to take place of $f$ in our construction \eqref{FNNopertors2} and derive approximation rates for $\alpha\geq 1$ by using the approaches in this paper and \cite{chui2019deep}. The problem is, however, that the feature qualifications requires two types of Berstein-type  inequalities (Lemma \ref{lemma1} and Lemma \ref{lemma2} below) that are difficult to verify for the deepened neural networks. We will keep studying and report the progress in a future study.
\end{remark}

\section{Stepping Stone: Lower Bound for Approximation via Bernstein Inequalities}\label{sec5}
To prove the results in Section 4, we need a crucial lower bound for univariate shallow net approximation constructed as follows.
For any univariate function $g^*$ defined on $[-1,1]$ and  $n\in\mathbb N$,
define
\begin{equation}\label{univariate-construction}
    G^*_{n,g^*}(t)=\sum_{k=0}^{4n}\beta_k\varphi(nt-k+2n),
\end{equation}
where
\begin{equation}\label{xishu}
\beta_j=\left\{\begin{array}{ll}
g^*(-1), &0\leq k\leq n-1;\\
g^*\left(\frac{k-2n}n\right), & n\leq k\leq 3n;\\
g^*(1), & 3n+1\leq k\leq 4n.
\end{array}\right.
\end{equation}
The following theorem describes the lower bound for approximation by $G_{n,g^*}^*$.

\begin{theorem}\label{theorem3}
For $g^*\in L^\infty([-1,1])$,  there holds
\begin{equation}\label{inverse-inequaty-sec2}
\omega\left(g^*,\frac1n\right)+\frac1n\|g^*\|_{L^\infty([-1,1])}\leq \frac Cn\sum^n_{k=1}\|G_{k,g^*}^*-g^*\|_{L^\infty([-1,1])},
\end{equation}
where $C$ represents a positive constant independent of $n$, but its value is generally different at different places.
\end{theorem}

To prove Theorem \ref{theorem3}, we need several auxiliary lemmas. Crucial ones  are two Bernstein-type inequalities.


\begin{lemma}\label{lemma1}
The inequality
$$
\left|\frac{\rm d}{{\rm d} t}G^*_{n,g^*}(t)\right|\leq C n\|g^*\|_{L^\infty([-1,1])}
$$
holds for any $g^*\in L^\infty([-1, 1])$ and $t\in [-1,1]$.
\end{lemma}

\begin{proof} Due to the definition of $G^*_{n,g^*}$,
it suffices  to prove
\begin{equation}\label{lemma1-1}
\sum_{j=0}^{4n}\left|\frac{\textrm d}{\textrm d x}\varphi(nt-j+2n)\right|\leq Cn.
\end{equation}
A simple computation gives
$$
\varphi^{\prime}(t)
=-\frac{(\textrm e^2-1)(\textrm e^{t-1}-\textrm e^{-t-1})}{2\textrm e^2(1+\textrm e^{t-1})^2(1+\textrm e^{-t-1})^2},
$$
which implies that for $t\ge0$, there holds
$$
\left|\varphi^{\prime}(t)\right|
\leq\frac{\textrm e^2-1}{2\textrm e^2}\frac {\textrm e^{t-1}}{\left(\textrm e^{t-1}\right)^2}
=\frac{\textrm e^2-1}{2\textrm e}\textrm e^{-t},
$$
and when $t<0$, 
$$
\left|\varphi^{\prime}(t)\right|=\frac{\textrm e^2-1}{2\textrm e^2}\frac {\textrm e^{-t-1}-\textrm e^{t-1}}{(1+\textrm e^{t-1})^2(1+\textrm e^{-t-1})^2}
\leq\frac{\textrm e^2-1}{2\textrm e^2}\frac {\textrm e^{-t-1}}{\left(\textrm e^{-t-1}\right)^2}
=\frac{\textrm e^2-1}{2\textrm e}\textrm e^{t}.
$$
Combining the two cases obtains the following estimate:
\begin{equation}\label{lemma1-2}
\left|\varphi^{\prime}(t)\right|\leq \frac{\textrm e^2-1}{2\textrm e}\textrm e^{-|t|}, \;  \; t\in \mathbb R.
\end{equation}
From \eqref{lemma1-1} and \eqref{lemma1-2}, it follows that
\begin{equation*}
\begin{aligned}
&\sum_{j=0}^{4n}\left|\frac{\textrm d}{\textrm d t}\varphi(nt-j+2n)\right|
\leq C n\sum_{j=0}^{4n}\textrm e^{-|nt-j+2n|}\\
&=Cn\left(\sum_{j\in\left\{j:\left|t-\frac{j-2}n\right|< \frac1n\right\}}\textrm e^{-|nt-j+2n|}
+\sum_{j\in\left\{j:\left|t-\frac{j-2}n\right|\geq \frac1n\right\}}\textrm e^{-|nt-j+2n|}\right).
\end{aligned}
\end{equation*}
Writing ${\mathbb E}_k:=\left\{j: \frac kn\leq\left|t-\frac{j-2}n\right|<\frac{k+1}n\right\}$, and
denoting $|\mathbb E_k|$ the cardinal number of the set $\mathbb E_k$, then the right side of the above equation is not greater than
\begin{eqnarray*}
\begin{aligned}
& Cn\sum_{j\in\left\{j:\left|t-\frac{j-2}n\right|< \frac1n\right\}}\textrm e^{-|nt-j+2n|}+Cn\sum_{k=1}^{+\infty}\sum_{j\in \mathbb E_k}\textrm e^{-|nt-j+2n|}\\
&\leq Cn\left(1+\sum_{k=1}^{+\infty}|\mathbb E_k|{\textrm e}^{-k}\right)
\leq Cn\left(1+\sum_{k=1}^{+\infty}(k+1){\textrm e}^{-k}\right)\\
&\leq Cn.
\end{aligned}
\end{eqnarray*}
This completes the proof of Lemma \ref{lemma1}.
\end{proof}

\begin{lemma}\label{lemma2}
Let $g^*\in L^\infty([-1,1])$ with $(g^*)^{\prime}\in C[-1, 1]$,
$$
\left|\frac{\rm d}{{\rm d} t}G_{n,g^*}^*(t) \right|\leq C \left(\|(g^*)^{\prime}\|_{L^\infty([-1,1])}+ n{\rm e}^{-n}\|g^*\|_{L^\infty([-1,1])}\right)
$$
holds for any $t\in [-1,1]$.
\end{lemma}
\begin{proof}
Extending $g^*$
to
$(-\infty,+\infty)$ by
\begin{equation}\label{extendF-aaa}
\mathcal{F}_1(t)=\left\{
            \begin{array}{lll}
            g^*(-1),&t\in (-\infty,-1);\\
            g^*(t) ,&t\in [-1,1]; \\
            g^*(1),&t\in (1,+\infty).
            \end{array}
                    \right.
\end{equation}
we have $\mathcal{F}_1 \in C(\mathbb R)$, and
$
\|\mathcal{F}_1\|_{L^\infty(\mathbb R)}=\|g^*\|_{L^\infty([-1,1]}.
 $
Due to Lemma \ref{Lemma:property-mean} (III),
we have
$$
\sum^{+\infty}_{j=-\infty}\frac{\textrm d}{\textrm dx}\varphi(nt-j)=0
$$
for any $t\in (-\infty,+\infty)$ and $n\in \mathbb N_+$,
which implies
\begin{eqnarray*}
\begin{aligned}
\left|\frac{\textrm d}{{\textrm d} t}G^*_{n,g^*}(t)\right|
&=\left|\sum^{+2n}_{i=-2n}g^*\left(\frac in\right)\frac{\textrm d}{\textrm dt}\varphi(nt-i)\right|\\
&=\left|\sum_{i=-\infty}^{+\infty}\left(\mathcal{F}_1\left(\frac in\right)-g^*(t)\right)\frac{\textrm d}{\textrm dt}\varphi(nt-i)-\sum_{|i|\ge 2n+1}\mathcal{F}_1\left(\frac in\right)\frac{\textrm d}{\textrm dt}\varphi(nt-i)\right|\\
&\leq\sum^{+\infty}_{i=-\infty}n\left|\left(\mathcal{F}_1\left(\frac in\right)-g^*(t)\right)\right|\left|\varphi^{\prime}(nt-i)\right|
+n\|g^*\|_{L^\infty([-1,1])}\sum_{|i|\ge 2n+1}\left|\varphi^{\prime}(nt-i)\right|\\
&:=\Sigma_1+\Sigma_2.
\end{aligned}
\end{eqnarray*}
For $\Sigma_2$, a simple computation achieves
$$
\varphi^{\prime}(t)
=\frac {\textrm e^{-t-1}-\textrm e^{t-1}}{(1+\textrm e^{t-1})(1+\textrm e^{-t-1})}\cdot \varphi(t),
$$
which follows
$$
\Sigma_2\leq n\|f\|_\infty\sum_{|i|\ge 2n+1}\varphi(nx-i).
$$
So, from the estimates of $\Delta_2$ in \cite{Chen2009The} (see pp.762 in \cite{Chen2009The}), we have
\begin{equation}\label{lemma2-2}
\Sigma_2\leq 3n\textrm e^{-n}\|g^*\|_{L^\infty([-1,1])}.
\end{equation}
Next, we estimate $\Sigma_1$.
\begin{eqnarray*}
\begin{aligned}
\Sigma_1=&\left(\sum^{2n}_{i=-2n}+\sum_{|i|\ge 2n+1}\right)n\left|\left(\mathcal{F}_1\left(\frac in\right)-g^*(t)\right)\right|\left|\varphi^{\prime}(nt-i)\right|\\
=&\sum^{2n}_{i=-2n}n\left|\left(g^*\left(\frac in\right)-g^*(t)\right)\right|\left|\varphi^{\prime}(nt-i)\right|
+\sum_{i=2n+1}^{+\infty}n\left|\left(g^*(1)-g^*(t)\right)\right|\left|\varphi^{\prime}(nt-i)\right|\\
&+\sum_{i=-\infty} ^{-2n-1}n\left|\left(g^*(-1)-g^*(t)\right)\right|\left|\varphi^{\prime}(nt-i)\right|\\
\leq & n\|(g^*)^{\prime}\|_{L^\infty([-1,1])}\sum^{2n}_{i=-2n}\left|\frac in-x\right|\left|\varphi^{\prime}(nt-i)\right|
+n\|(g^*)^{\prime}\|_{L^\infty([-1,1])}\sum_{i=2n+1}^{+\infty}
\left|1-t\right|\left|\varphi^{\prime}(nt-i)\right|\\
&+n\|(g^*)^{\prime}\|_{L^\infty([-1,1])}\sum_{i=-\infty} ^{-2n-1}\left|-1-t\right|\left|\varphi^{\prime}(nt-i)\right|\\
\leq & n\|(g^*)^{\prime}\|_{L^\infty([-1,1])}\sum^{2n}_{i=-2n}\left|\frac in-x\right|\left|\varphi^{\prime}(nt-i)\right|
+2n\|(g^*)^{\prime}\|_{L^\infty([-1,1])}\sum_{|i|\ge 2n+1}
\left|\varphi^{\prime}(nt-i)\right|\\
=&\Sigma_{11}+\Sigma_{12}.
\end{aligned}
\end{eqnarray*}
Similar to $\Sigma_2$, we obtain
\begin{equation}\label{Sigma12}
\Sigma_{12}\leq 6n{\textrm e}^{-n}\|(g^*)^{\prime}\|_{L^\infty([-1,1])}\leq 6\|(g^*)^{\prime}\|_{L^\infty([-1,1])}.
\end{equation}
Applying the inequality \eqref{lemma1-2}, we have
\begin{eqnarray*}
\Sigma_{11}&\leq& Cn\|(g^*)^{\prime}\|_{L^\infty([-1,1])}\left(\sum_{|\frac in-t|<\frac1n}\left|\frac in-t\right|{\textrm e}^{-|nt-i|}+\sum_{|\frac in-x|\ge\frac1n}\left|\frac in-t\right|{\textrm e}^{-|nt-i|}\right).
\end{eqnarray*}
So, a similar method to Lemma \ref{lemma1} implies
\begin{equation}\label{Sigma11}
\Sigma_{11}\leq C\|(g^*)^{\prime}\|_{L^\infty([-1,1])}.
\end{equation}
Integrating \eqref{Sigma12} and \eqref{Sigma11}, it follows that
 \begin{equation}\label{lemma2-1}
\Sigma_{1}\leq C\|(g^*)^{\prime}\|_{L^\infty([-1,1])}.
\end{equation}
Combining the \eqref{lemma2-1} and \eqref{lemma2-2}, we complete
the proof of Lemma \ref{lemma2}.
\end{proof}



The next two lemmas focus on the
relationship between $K$-functional and modulus of continuity, which is standard in the approximation theory literature \cite{Ditzian1987Moduli}. For the sake of brevity in proving Theorem \ref{theorem3}, we provide the proof of Lemma \ref{lemma7} in the next section. 
Let $g^*\in L^\infty([-1,1])$ and
  $h_2>0$. Define
 \begin{equation}\label{k-fun-def-1}
 K(g^*,t):=\inf_{g,g^{\prime} \in L^\infty([-1,1])}\left\{\|g^*-g\|_{L^\infty([-1,1])}+t\|g^{\prime}\|_{L^\infty([-1,1])}\right\},
 \end{equation}
 and
 \begin{equation}\label{k-fun-def-2}
 \widetilde{K}(g^*,,h_2):=\inf_{g,g^{\prime} \in L^\infty([-1,1])}\left\{\|g^*-g\|_{L^\infty([-1,1])}+\|g'\|_{L^\infty([-1,1])}+h_2\|g\|_{L^\infty([-1,1])}
 \right\}.
 \end{equation}
The following two lemmas are needed in our proof.
\begin{lemma}\label{lemma6}
The weak equivalence relation
$$
	\omega(g^*,h)\asymp K(g^*,h)
$$
holds for any $g^*\in L^\infty([-1,1])$ and $h>0$.
\end{lemma}

\begin{lemma}\label{lemma7}
For any $g^*\in L^\infty([-1,1])$ and $,h_2>0$, there holds
$$
\widetilde{K}(g^*,,h_2)\asymp\min(1,h_2)\|g^*\|_{L^\infty([-1,1])}+K(g^*,h_1).
$$
\end{lemma}

Our final two lemmas refer to recursive relationships  of several sequences and can be regarded as an extension of two corresponding results in \cite{van1986steckin} in terms that  we omit the requirements that the first terms of sequences $\{\sigma_n\}$, $\{\mu_n\}$, and $\{\nu_n\}$ are all zero. We also provide their proofs in the next section for the sake of completeness.

\begin{lemma}\label{lemma3-sec2}
Suppose  $\{\sigma_n\}$ and $\{\tau_n\}$ are all non-negative sequences. If for $p>0$,
$1\leq k\leq n$, and $n\in \mathbb N$, there holds
\begin{equation}\label{lemma3-sec2-1}
\sigma_n\leq \left(\frac kn\right)^p\sigma_k+\tau_k,
\end{equation}
then
\begin{equation}\label{lemma3-sec2-2}
\sigma_n\leq 4^pn^{-p}\sum_{k=1}^nk^{p-1}\tau_k.
\end{equation}
\end{lemma}

\begin{lemma} \label{lemma4-sec2}
Suppose $\{\mu_n\}$, $\{\nu_n\}$, and $\{\psi_n\}$ are all non-negative sequences.
If for $0<r<s$, $
1\leq k\leq n$, and $n\in \mathbb N$, there hold
\begin{equation}\label{lemma4-sec2-1}
\mu_n\leq \left(\frac kn\right)^r\mu_k+\nu_k+\psi_k,
\end{equation}
and
\begin{equation}\label{lemma4-sec2-2}
\nu_n\leq \left(\frac kn\right)^s\nu_k+\psi_k,
\end{equation}
then
$$
\mu_n\leq Cn^{-1}\sum_{k=1}^nk^{r-1}\psi_k.
$$
\end{lemma}

With the help of the above six lemmas, we are in a position to prove Theorem \ref{theorem3}.

\begin{proof}[Proof of Theorem \ref{theorem3}]
  Let
$$
\tau_n=\|G^*_{n,g^*} -g^*\|_{L^\infty([-1,1])}, \;  \;
\sigma_n=\frac1n\|(G^*_{n,g^*})^{\prime} \|_{L^\infty([-1,1])}+\frac1n\|G^*_{n,g^*}\|_{L^\infty([-1,1])}.
$$
For $1\leq k\leq n$, it follows from Lemma \ref{lemma1} and Lemma \ref{lemma2}  that
\begin{eqnarray*}
\begin{aligned}
\sigma_n&=\frac1n\|(G^*_{n,g^*})^{\prime}(g^*-G_{k,g^*}^* +G_{k,g^*}^*)\|_{L^\infty([-1,1])}+\frac1n\|G^*_{n,g^*} \|_{L^\infty([-1,1])}\\
&\leq C \|g^*-G_{k,g^*}^* \|_{L^\infty([-1,1])}+\frac Cn\|(G_{k,g^*}^*)^{\prime} \|_{L^\infty([-1,1])}+\frac1n\|G^*_{n,g^*}\|_{L^\infty([-1,1])}\\
&=C\left(\frac kn\right)\sigma_k+C\tau_k.
\end{aligned}
\end{eqnarray*}
So, applying Lemma \ref{lemma3-sec2} follows
$\sigma_n\leq \frac Cn\sum_{k=1}^n\tau_k,
$
that is,
\begin{equation}\label{sec3-11}
\frac1n\|(G^*_{n,g^*})^{\prime} \|_{L^\infty([-1,1])}+\frac1n\|G^*_{n,g^*}\|_{L^\infty([-1,1])}\leq \frac Cn\sum_{k=1}^n\|G_{k,g^*}^*-g^*\|_{L^\infty([-1,1])}.
\end{equation}
For $n\ge 2$, there exists an $m\in\mathbb N$, such that $\frac n2\leq m\leq n$ and for
$1\leq k\leq n$ there holds
$$
\|G_{m,g^*}^*-g^*\|_{L^\infty([-1,1])}\leq \|G_{k,g^*}^*-g^*\|_{L^\infty([-1,1])},
$$
which follows
\begin{equation}\label{sec3-12}
\|G_{m,g^*}^*-g^*\|_{L^\infty([-1,1])}\leq \frac4n\sum^n_{k=\frac n2}\|G_{k,g^*}^*-g^*\|_{L^\infty([-1,1])}.
\end{equation}
Therefore, writing
$\widetilde{K}(g^*,h):=\widetilde{K}(g^*,h,h)$ and taking $g=G_{m,g^*}^*$, it follows from
   the definition \eqref{k-fun-def-2} that
\begin{equation}\label{sec3-13}
\widetilde{K}\left(g^*,\frac1n\right)\leq \|G_{m,g^*}^*-g^*\|_{L^\infty([-1,1])}+\frac1n\|(G_{m,g^*}^*)^{\prime} \|_{L^\infty([-1,1])}+\frac1n\|G_{m,g^*}^*\|_{L^\infty([-1,1])}.
\end{equation}
Recalling \eqref{sec3-11}, we achieve
\begin{equation}\label{sec3-14}
\frac1n\|(G_{m,g^*}^*)^{\prime} \|_{L^\infty([-1,1])}+\frac1n\|G_{m,g^*}^*\|_{L^\infty([-1,1])} \leq \frac Cn\sum_{k=1}^m\|G_{k,g^*}^*-g^*\|_{L^\infty([-1,1])}.
\end{equation}
Integrating \eqref{sec3-12}, \eqref{sec3-13}, and \eqref{sec3-14} implies
$$
\widetilde K\left(g^*,\frac1n\right)\leq \frac Cn \sum_{k=1}^n\|G_{k,g^*}^* -g^*\|_{L^\infty([-1,1])},
$$
which follows the inverse inequality from Lemma \ref{lemma7}:
$$
\omega\left(g^*,\frac1n\right)+\frac1n\|g^*\|_{L^\infty([-1,1])}\leq \frac Cn\sum^n_{k=1}\|G_{k,g^*}^*-g^*\|_{L^\infty([-1,1])}
$$
and completes the proof of Theorem \ref{theorem3}.
\end{proof}


\section{Proofs}\label{sec6}
This section devotes to the proof of our theoretical results.

We at first prove Lemma \ref{Lemma:lip} as follows.
\begin{proof}[Proof of Lemma \ref{Lemma:lip}]
If $f\in\mathcal R_\tau$ with $g_f\in Lip^{\alpha,c,\nu}([0,1])$, then for any $x,x'\in\mathbb B^d$ there holds
\begin{eqnarray*}
\begin{aligned}
      |f(x)-f(x')|
      &\leq  |f(x)-g_f(\|x\|_2^2)|+|f(x')-g_f(\|x'\|_2^2) |+|g_f(\|x\|_2^2)-g_f(\|x'\|_2^2)|\\
       &\leq
      2\tau+\nu+c|\|x\|_2^2-\|x'\|_2^2|^\alpha.
      \end{aligned}
      \end{eqnarray*}
      
But
\begin{eqnarray*}
\begin{aligned}
     |\|x\|_2^2-\|x'\|_2^2|
     &=
     \sum_{\ell=1}^d |(x^{(\ell)}+(x')^{(\ell)})(x^{(\ell)}-(x')^{(\ell)})|
   \leq
   2\sum_{\ell=1}^d|x^{(\ell)}-(x')^{(\ell)}|\\
   &\leq
   2\sqrt{d}\|x-x'\|_2.
 \end{aligned}  
\end{eqnarray*}

Therefore,
$$
   |f(x)-f(x')|\leq 2\tau+\nu+c2^\alpha d^{\alpha/2}\|x-x'\|_2^\alpha,
$$
implying $f\in Lip^{\alpha,2^\alpha d^{\alpha/2},2\tau+\nu}$.  For any $t\in[0,1]$, define $x_t=(\overbrace{0,\dots,0}^{d-1},t)^T\in\mathbb B^d$,
if $f\in\mathcal R_\tau\cap Lip^{\alpha,c,\nu'}$ for some $c',\nu'>0$, then
\begin{eqnarray*}
\begin{aligned}
     |g_f(t)-g_f(t')|
     &=
     |g_f(\|x_t\|_2^2)-g_f(\|x_{t'}\|_2^2)|\\
     &\leq
     |g_f(\|x_t\|_2^2)-f(x_t)|+|g_f(\|x_{t'}\|_2^2)-f(x_{t'})|+|f(x_t)-f(x_{t'})|\\
     &\leq
     2\tau+\nu'+c\|x_t-x_{t'}\|_2^\alpha
     =2\tau+\nu'+c|t-t'|^\alpha,
\end{aligned}
\end{eqnarray*}
showing $g\in Lip^{\alpha,c,2\tau+\nu'}$.
This completes the proof of Lemma \ref{Lemma:lip}.
\end{proof}

For any $f\in  \mathcal R_\tau$ with  $\tau\geq 0$,  and $g_f$ satisfying \eqref{near-radial}.
Define $g_f^*(t)=g_f(2t-1)$, then $g_f(t)=g_f^*((t+1)/2)$, $g_f^*\in L^\infty([-1,1])$  and
\begin{equation}\label{near-radial-scaling}
    |f(x)-g_f^*(\widetilde{\|x\|_2^2})|\leq \tau,
\end{equation}
where $\widetilde{t}=2t-1$ for any $t\in[0,1]$.
Due to the definition of $g^*_f$, it is obvious that
$
    \omega(g^*_f,h)=\omega(g_f,2h).
$
Then, we proceed the proof of Theorem \ref{directTheorem} by using Lemma \ref{Lemma:lip}, Lemma \ref{Lemma:polynomial for nn} and Lemma \ref{Lemma:property-mean}.

\begin{proof}[Proof of Theorem \ref{directTheorem}]
Since $g^*_f:[-1,1]$ can be extended
to
$(-\infty,+\infty)$ as the manner:
\begin{equation}\label{extendF}
\mathcal{F}_f(t)=\left\{
            \begin{array}{lll}
            g^*_f(-1),&t\in (-\infty,-1);\\
            g^*_f(t) ,&t\in [-1,1]; \\
            g^*_f(1),&t\in (1,+\infty).
            \end{array}
                    \right.
\end{equation}
Clearly,
$
\|\mathcal{F}_f\|_{L^\infty(\mathbb R)}\leq \|g_f^*\|_{L^\infty([-1,1])}$ and $
\omega(\mathcal{F}_f, h)\leq \omega(g^*_f,h)=\omega(g_f,2h)$.
Due to \eqref{near-radial-scaling} and \eqref{extendF},
we have from \eqref{l2norm-app-rate} that
\begin{eqnarray}\label{proof.radial-1}
\begin{split}
     |f(x)-\mathcal F_f(\mathcal N_{0,t^2,\varepsilon}(x))|
     &\leq
 |f(x)-g_f^*(\|\widetilde{\|x\|_2^2})|+
  |\mathcal F_f(\|\widetilde{\|x\|_2^2})-\mathcal F_f(\widetilde{\mathcal N_{0,t^2,\varepsilon}(x))}|\\
  &\leq \tau+\omega(\mathcal{F}_f, d\varepsilon)
  \leq
  \tau+2d\omega(g_f,\varepsilon).
\end{split}
\end{eqnarray}
Based on
 \eqref{FNNopertors2}, $f\in\mathcal R_\tau$ and the definition of $g^*_f$, $\mathcal F_f$,
we have from \eqref{univariate-construction} and Lemma \ref{Lemma:property-mean} (III) that  for any $x\in\mathbb B^d$, there holds
\begin{eqnarray}\label{First-1122}
    |G_{n,f,\varepsilon}(x)-G^*_{n,g_f^*}(\mathcal N_{0,t^2,\varepsilon}(x))|
    \leq \tau \max_{t\in\mathbb R}\sum_{k={0}}^{4n}\varphi(nt-k+2n)\leq \tau.
\end{eqnarray}
We then focus on $G^*_{n,g_f^*}(\mathcal N_{0,t^2,\varepsilon}(x))$ and rewrite it as
\begin{eqnarray*}
\begin{aligned}
  G^*_{n,g_f^*}(\mathcal N_{0,t^2,\varepsilon}(x))
  &=
  \sum_{i=-2n}^{-n-1}g^*_f(-1)\varphi(nN_{0,t^2,\varepsilon}(x)-i)+\sum_{i=-n}^{n}g^*_f\left(\frac in\right)\varphi(nN_{0,t^2,\varepsilon}(x)-i)\\
  &+
\sum_{i=n+1}^{2n}g^*_f(1)\varphi(nN_{0,t^2,\varepsilon}(x)-i).
\end{aligned}
\end{eqnarray*}
 Therefore, Lemma \ref{Lemma:property-mean} (III) implies that
\begin{equation}\label{sec2-1}
\begin{aligned}
&|\mathcal F_f(N_{0,t^2,\varepsilon}(x))-G^*_{n,g_f^*}(\mathcal N_{0,t^2,\varepsilon}(x))|\\
&=\left|\sum_{i=-\infty}^{+\infty}\mathcal F_f(N_{0,t^2,\varepsilon}(x))\varphi(nN_{0,t^2,\varepsilon}(x)-i)-\sum_{i=-2n}^{2n}
\mathcal{F}_f\left(\frac in\right)\varphi(nN_{0,t^2,\varepsilon}(x)-i)\right| \\
&\leq\sum_{i=-2n}^{2n}\left|\mathcal{F}_f(N_{0,t^2,\varepsilon}(x))-\mathcal{F}_f\left(\frac in\right)\right|\varphi(nN_{0,t^2,\varepsilon}(x)-i)\\
&+\sum_{|i|\geq {2n+1}}|\mathcal{F}_f(N_{0,t^2,\varepsilon}(x))|\varphi(nN_{0,t^2,\varepsilon}(x)-i)\\
&:=\Delta_1+\Delta_2.
\end{aligned}
\end{equation}
We first estimate $\Delta_2$. Direct computation yields
$$
\varphi(t)= \frac {\textrm e^2-1}{2\textrm e^2(1+\textrm e^{t-1})(1+\textrm e^{-t-1})},
$$
which implies for $t\in (-\infty,+\infty)$,
\begin{equation}\label{sec2-10}
\varphi(t)\leq\frac {\textrm e^2-1}{2\textrm e}\textrm e^{-|t|}.
\end{equation}
Moreover, when $|t|\leq 1$
we have $|nt-i|\geq {n+1}$ for $|i|\geq {2n+1}$.
Hence,
\begin{equation}\label{sec2-2}
\begin{aligned}
\Delta_2&\leq  \|g_f^*\|_{L^\infty([-1,1])}\sum_{|i|\geq {2n+1}}\varphi(nN_{0,t^2,\varepsilon}(x)-i)\\
&\leq\|g_f^*\|_{L^\infty([-1,1])}\int_{n}^{+\infty}
 \frac {\textrm e^2-1}{{\textrm e^2}(1+\textrm e^{t-1})(1+\textrm e^{-t-1})}\textrm dt\\
&\leq \frac {\textrm e^2-1}{\textrm {e}}\textrm {e}^{-n}\|g_f^*\|_{L^\infty([-1,1])}
\leq 3\textrm {e}^{-n}\|g_f^*\|_{L^\infty([-1,1])}
\end{aligned}
\end{equation}
Next, we estimate $\Delta_1$.
Writing $\mathbb J_j:=\{i: \frac jn\leq\left|N_{0,t^2,\varepsilon}(x)-\frac in\right|<\frac{j+1}n\}$ and
denoting $|\mathbb J_j|$ the cardinal number of the set $\mathbb J_j$. Clearly, $|\mathbb J_j|\leq 3$. So,
\begin{equation}\label{sec2-3}
\begin{aligned}
\Delta_1
&= \sum_{|N_{0,t^2,\varepsilon}(N_{0,t^2,\varepsilon}(x))-\frac in|\leq \frac 1{{n}}}\left|\mathcal{F}_f(N_{0,t^2,\varepsilon}(x))-\mathcal{F}\left(\frac in\right)\right|\varphi(nN_{0,t^2,\varepsilon}(x)-i)\\
&+\sum_{|N_{0,t^2,\varepsilon}(x)-\frac in|>\frac 1{{n}}}\left|\mathcal{F}_f(N_{0,t^2,\varepsilon}(x))-\mathcal{F}\left(\frac in\right)\right|\varphi(nN_{0,t^2,\varepsilon}(x)-i)\\
&\leq \omega\left(\mathcal{F}_f,\frac 1{{n}}\right)\sum_{i=-\infty}^{+\infty}\varphi(nN_{0,t^2,\varepsilon}(x)-i)\\
&+\sum_{j=1}^{+\infty}\sum_{i\in {\mathbb J}_j}\left|\mathcal{F}_f(x)-\mathcal{F}_f\left(\frac in\right)\right|\varphi(nN_{0,t^2,\varepsilon}(x)-i)\\
&\leq \omega\left(\mathcal{F}_f,\frac 1{{n}}\right)
+\sum_{j=1}^{+\infty}\omega\left(\mathcal{F}_f,\frac{j+1}n\right)\sum_{i\in {\mathbb J}_j}\varphi(nN_{0,t^2,\varepsilon}(x)-i).
\end{aligned}
\end{equation}
Using the facts $\omega(\mathcal{F}_f, h)\leq \omega(g^*_f,h)\leq \omega(g_f,2h)$,  $\omega(g_f,kh)\leq k\omega(g_f,h)$ ($ k\in\mathbb N_+$), and the inequality \eqref{sec2-10}, the right side of the above inequality \eqref{sec2-3} does not exceed
\begin{equation}\label{sec2-4}
\begin{aligned}
&2\omega\left(g_f,\frac 1{{n}}\right)\left(1+\frac {\textrm e^2-1}{\textrm {2e}}\sum_{j=1}^{+\infty}(j+1)\sum_{i\in {\mathbb J}_j}\textrm e^{-|nN_{0,t^2,\varepsilon}(x)-i|}\right)\\
&\leq 2\omega\left(g_f,\frac 1{{n}}\right)\left(1+\frac {\textrm e^2-1}{\textrm {2e}}\sum_{j=1}^{+\infty}(j+1)\textrm e^{-j}\right).
\end{aligned}
\end{equation}
A simple computation follows the inequity:
\begin{equation}\label{sec2-4a}
\sum_{j=1}^{+\infty}(j+1)\textrm e^{-j}\leq 3\textrm{e}^{-1}.
\end{equation}
So, collecting \eqref{sec2-3}, \eqref{sec2-4}, and \eqref{sec2-4a} obtains
\begin{eqnarray}\label{sec2-5}
\Delta_1\leq \frac{10\textrm e^2-3}{ \textrm e^2}\omega\left(g_f,\frac 1{{n}}\right).
\end{eqnarray}
Amalgamating \eqref{proof.radial-1},
\eqref{sec2-1}, \eqref{sec2-2} and \eqref{sec2-5}, we obtain
$$
  |\mathcal F_f(N_{0,t^2,\varepsilon}(x))-G^*_{n,g_f^*}(\mathcal N_{0,t^2,\varepsilon}(x))|\leq 3\textrm {e}^{-n}\|f\|_{L^\infty(\mathbb B^d)}+\frac{10\textrm e^2-3}{ \textrm e^2}\omega\left(g_f,\frac 1{{n}}\right).
$$
Combining the above estimate with \eqref{proof.radial-1} and \eqref{First-1122}, we complete the proof
  the proof of Theorem \ref{directTheorem}.
\end{proof}

 We then use Theorem \ref{theorem3} to prove Theorem \ref{Theorem:qualification-class}.

\begin{proof}[Proof of Theorem \ref{Theorem:qualification-class}]
If $f\in\mathcal R_\tau$ and
 $$
\|G_{n,f,\varepsilon} -f\|_{L^\infty(\mathbb B^d)}
\leq c_1n^{-\alpha}
$$
for $0<\alpha\leq1$ and $\tau\geq 0$,
then  we have from \eqref{near-radial-scaling} and \eqref{univariate-construction} that
\begin{eqnarray*}
\begin{aligned}
   &\max_{-1\leq t\leq 1}|G^*_{n,g^*_f}(t)-g^*_f(t)|
    \leq
   \max_{x\in\mathbb B^d}|G^*_{n,g^*_f}(\|x\|_2^2)-g^*_f(\widetilde{\|x\|_2^2})| \\
   &\leq
   \max_{x\in\mathbb B^d}| G^*_{n,g^*_f}(\|x\|_2^2)-G_{n,f,\varepsilon}(x)|
   +
   \max_{x\in\mathbb B^d}|G_{n,f,\varepsilon}(x)-f(x)|
   +\max_{x\in\mathbb B^d}|f(x)-g^*(\widetilde{\|x\|_2^2})|\\
   &\leq
   \max_{x\in\mathbb B^d}| G^*_{n,g^*_f}(\|x\|_2^2)-G_{n,f,\varepsilon}(x)|+c_1n^{-\alpha} +\tau.
\end{aligned}
\end{eqnarray*}
Due to \eqref{l2norm-app-rate}, \eqref{FNNopertors2}, \eqref{univariate-construction}, Lemma \ref{lemma1}  and Lemma \ref{Lemma:property-mean} (III),
 we have
\begin{eqnarray*}
\begin{aligned}
      | G^*_{n,g^*_f}(\|x\|_2^2)-G_{n,f,\varepsilon}(x)|
      &\leq
       | G^*_{n,g^*_f}(\|x\|_2^2)-G^*_{n,g^*_f}(N_{0,t^2,\varepsilon}(x))|+
       |G^*_{n,g^*_f}(N_{0,t^2,\varepsilon}(x))-
       G_{n,f,\varepsilon}(x)|\\
       &\leq
     C n \|g_f^*\|_{L^\infty([-1,1])} |\|x\|_2^2-N_{0,t^2,\varepsilon}(x)|+\tau\\
      &\leq
       C dn \|g_f^*\|_{L^\infty([-1,1])} \varepsilon+\tau.
       \end{aligned}
\end{eqnarray*}
Therefore, $0<\varepsilon\leq n^{-1-\alpha}$ and $0\leq \tau\leq n^{-\alpha}$ implies
$$
   \max_{-1\leq t\leq 1}|G^*_{n,g^*_f}(t)-g^*_f(t)|\leq  Cn^{-\alpha}.
$$
Then it follows from
  \eqref{inverse-inequaty-sec2} that
\begin{equation*}
	\omega\left(g^*_f,\frac{1}{n}\right)\leq \frac Cn\sum^n_{k=1}k^{-\alpha}\leq Cn^{-\alpha}.
\end{equation*}
 According to the monotonicity of the continuous modulus $\omega(f,h_1)\leq \omega(f,h_2) (0<h_1\leq h_2)$, for any $0<h<1$, there exists a positive integer $n$ such that $\frac1{2n}\leq h<\frac 1n$ and
$$
\omega\left(g^*_f,h\right)\leq \omega\left(g^*_f,\frac{1}{n}\right)={\mathcal O}(n^{-\alpha})={\mathcal O}(t^{\alpha}),
$$
i.e., $g^*_f \in Lip^{\alpha,c}([-1,1])$, and consequently  $g_f\in Lip^{\alpha,c}([0,1]).$ This together with Lemma \ref{Lemma:lip} implies
 $f\in \mathcal R_\tau\cap Lip^{\alpha,c,2\tau}$ for some $c>0$ independent of $n,\varepsilon,\tau$ and verifies $Lip^{\alpha,c,2\tau} \overset{c_1n^{-\alpha}}{\propto}\mathcal G_{n,\varepsilon}|_{\mathcal R_\tau} $. The proof of Theorem \ref{Theorem:qualification-class} is completed.
\end{proof}

To prove Theorem \ref{Theorem:qualification-class-2}, we need the following lemma that can be deduced from the proof of Theorem \ref{directTheorem} directly.
\begin{lemma}\label{Lemma:direct-univariate}
Let $G_{n,g^*}$ be defined by \eqref{univariate-construction}.  If $g^*\in Lip^{\alpha,c,\nu}([-1,1])$ for $ 0<\alpha\leq 1$, $c,\nu>0$, then
$$
    \|g^*-G_{n,g^*}\|_{L^\infty([-1,1])}\leq Cn^{-\alpha}+\nu.
$$
\end{lemma}

We then prove Theorem \ref{Theorem:qualification-class-2} as follows.

\begin{proof}[Proof of Theorem \ref{Theorem:qualification-class-2}]
For any $t\in[0,1]$, define
$g^*_f(t)=f(\overbrace{0,\dots,0}^{d-1},t)$ and consequently $g_f( {t})=g^*_f((t+1)/2)$. Noting  $f\in Lip^{\alpha,c,\nu}$, we get from the proof of  Lemma \ref{Lemma:lip} that $g_f\in Lip^{\alpha,c,2\tau+\nu}([0,1])$ and $g_f^*\in Lip^{\alpha,2c,2\tau+\nu}([-1,1])$.
For any $n\in\mathbb N$ and $0\leq k\leq 4n$, denote
$$
  t_k=\left\{\begin{array}{ll}
    -1, &0\leq k\leq n-1;\\
 \frac{k-2n}{n} , & n\leq k\leq 3n;\\
1, & 3n+1\leq k\leq 4n.
\end{array}\right.
$$
We then have $g_f((t_k+1)/2)=g_f^*( {t_k})=f(\xi_k)$, where $\xi_k$ is give by \eqref{Def.xi}. In this way, we have from \eqref{univariate-construction} and \eqref{FNNopertors2}
that
$$
     G_{n,f,\varepsilon}(x)=G_{n,g_f^*}(N_{0,t^2,\varepsilon}(x))).
$$
Since
  $$
\|G_{n,f,\varepsilon} -f\|_{L^\infty(\mathbb B^d)}
\leq c_2n^{-\alpha},
$$
we have from $g_f^*\in Lip^{\alpha,2c,2\tau+\nu}([-1,1])$  and Lemma \ref{Lemma:direct-univariate} that
\begin{eqnarray*}
\begin{aligned}
     &\max_{x\in\mathbb B^d}|f(x)-g_f(\|x\|_2^2)|
     \leq
    \max_{x\in\mathbb B^d}|G_{n,f,\varepsilon}(x) -f(x)|
    +\max_{x\in\mathbb B^d}|G_{n,f,\varepsilon}(x)-g_f(\|x\|_2^2)|\\
    &\leq
    c_2n^{-\alpha} +\max_{x\in\mathbb B^d}|G^*_{n,g^*_f}(N_{0,t^2,\varepsilon}(x))-G_{n,g^*_f}^*(\|x\|_2^2)|
    +
    \max_{-1\leq t\leq 1}|
     g^*_f(t)-G_{n,g^*_f}^*(t)|\\
     &\leq
     (c_2+C)n^{-\alpha} +Cn\varepsilon+\nu.
     \end{aligned}
\end{eqnarray*}
Therefore, if $\varepsilon\leq n^{-1-\alpha}$ and $\nu\leq n^{-\alpha}$, we have
$$
   \max_{x\in\mathbb B^d}|f(x)-g_f(\|x\|_2^2)|
     \leq  Cn^{-\alpha},
$$
showing that $f\in\mathcal R_\tau$ with $\tau\geq Cn^{-\alpha}$. This completes the proof of Theorem \ref{Theorem:qualification-class-2}.
\end{proof}

Finally, we prove Lemma \ref{lemma7}, Lemma \ref{lemma3-sec2} and Lemma \ref{lemma4-sec2} by direct computation.

\begin{proof}[Proof of Lemma \ref{lemma7}]
On the one hand, taking $g=0$ in the definition of $\widetilde{K}(g^*,h_1,h_2)$ follows that for $g^*\in L^\infty([-1,1])$ and $h_1, h_2>0$
$$
\widetilde{K}(g^*,h_1,h_2)\leq \|g^*\|_{L^\infty([-1,1])}.
$$
On the other hand, for any $g\in L^\infty([-1,1])$ with $g^{\prime}\in L^\infty([-1,1])$, one has
\begin{eqnarray*}
\begin{aligned}
\widetilde{K}(g^*,h_1,h_2)&\leq \|f-g\|_{L^\infty([-1,1])}+h_1\|g^{\prime}\|_{L^\infty([-1,1])}+h_2\|g\|_{L^\infty([-1,1])}
 \\
 &\leq (1+h_2)\|g^*-g\|_{L^\infty([-1,1])}+h_1\|g^\prime\|_{L^\infty([-1,1])}+h_2\|g^*\|_{L^\infty([-1,1])}.
 \end{aligned}
\end{eqnarray*}
So, combining these two cases, we have
\begin{eqnarray*}
\begin{aligned}
\widetilde{K}(g^*,h_1,h_2)
 &\leq \min(1,h_2)\|g^*\|_{L^\infty([-1,1])}\\
 &+(1 +h_2)\chi_{(0,1)}(h_2)\left(\|g^*-g\|_{L^\infty([-1,1])}
 +\frac{h_1}{1+h_2}\|g^\prime\|_{L^\infty([-1,1])}\right),
\end{aligned}
\end{eqnarray*}
where $\chi_{(0,1)}(h_2)$ is characteristic function about $h_2$, namely, for $h_2\in (0,1)$, $\chi_{(0,1)}(h_2)=1$, while $h_2\notin (0,1)$, $\chi_{(0,1)}(h_2)=0$. Therefore, form the definition of K-functional \eqref{k-fun-def-1} and the above inequality, it follows that
\begin{eqnarray}\label{lemma7-1}
\begin{aligned}
\widetilde{K}(g^*,h_1,h_2)
 &\leq \min(1,h_2)\|g^*\|_{L^\infty([-1,1])}
 +(1 +h_2)\chi_{(0,1)}(h_2)K(g^*,h_1)\\
 &\leq 2\left(\min(1,h_2)\|g^*\|_{L^\infty([-1,1])}
 +K(g^*,h_1)\right).
 \end{aligned}
\end{eqnarray}
We then prove the inversion of inequality \eqref{lemma7-1}, i.e.,
 \begin{eqnarray}\label{lemma7-2}
\widetilde{K}(g^*,h_1,h_2)\geq \frac12
 \left(\min(1,h_2)\|g^*\|_{L^\infty([-1,1])}
 +K(g^*,h_1)\right).
\end{eqnarray}
In fact, recalling that $\|g\|_{L^\infty([-1,1])}\geq \|g^*\|_{L^\infty([-1,1])}-\|g^*-g\|_{L^\infty([-1,1])}$, we have for $0<h_2<1$ and any $g\in L^\infty([-1,1])$ with $g^{\prime}\in L^\infty([-1,1])$ that
\begin{eqnarray}\label{lemma7-3}
\begin{aligned}
K(g^*,h_1)&\leq \|g^*-g\|_{L^\infty([-1,1])}+h_1\|g^{\prime}\|_{L^\infty([-1,1])}+h_2\|g\|_{L^\infty([-1,1])}
-h_2\|g\|_{L^\infty([-1,1])}\\
&\leq (1+h_2)\|g^*-g\|_{L^\infty([-1,1])}+h_1\|g^{\prime}\|_{L^\infty([-1,1])}\\
&+h_2\|g\|_{L^\infty([-1,1])}
-h_2\|g^*\|_{L^\infty([-1,1])}\\
&\leq 2\|g^*-g\|_{L^\infty([-1,1])}+h_1\|g^{\prime}\|_{L^\infty([-1,1])}+h_2\|g\|_{L^\infty([-1,1])}
-h_2\|g^*\|_{L^\infty([-1,1])}.
\end{aligned}
\end{eqnarray}
For $h_2>1$, we have
\begin{eqnarray}\label{lemma7-4}
\begin{aligned}
K(g^*,h_1)&\leq \|g^*-g\|_{L^\infty([-1,1])}+h_1\|g^{\prime}\|_{L^\infty([-1,1])}+h_2\|g\|_{L^\infty([-1,1])}
-h_2\|g\|_{L^\infty([-1,1])}\\
&\leq \|g^*-g\|_{L^\infty([-1,1])}+h_1\|g^{\prime}\|_{L^\infty([-1,1])}+h_2\|g\|_{L^\infty([-1,1])}
-\|g\|_{L^\infty([-1,1])}\\
&\leq 2\|g^*-g\|_{L^\infty([-1,1])}+h_1\|g^{\prime}\|_{L^\infty([-1,1])}+h_2\|g\|_{L^\infty([-1,1])}
-\|g^*\|_{L^\infty([-1,1])}.
\end{aligned}
\end{eqnarray}
So, combining \eqref{lemma7-3} and \eqref{lemma7-4} implies
\eqref{lemma7-2}. The proof of Lemma \ref{lemma7} is completed.
\end{proof}

\begin{proof}[Proof of Lemma \ref{lemma3-sec2}]
Let $n\ge 2$ be given, and choose $N\in \mathbb N$ such that
\begin{equation}\label{App-1}
2^{N}\leq n< 2^{N+1}.
\end{equation}
Then from \eqref{App-1}, it follows that
$2^k\leq \frac{n}{2^{N-k}}< 2^{k+1}$ for $1\leq k\leq N$. Furthermore,
$$
2^{N-k}\leq \frac{n}{2^{N-k}}< 2^{N-k+1}, \  \ \textrm{for} \  1\leq k\leq N,
$$
that is,
$$
n2^{-(N-k)-1}< 2^{N-k}\leq n2^{-(N-k)},\  \ \textrm{for} \  1\leq k\leq N,
$$
which follows that for $0\leq k< N$
$$
n2^{-k-1}<2^k\leq n2^{-k}.
$$
Let
$$
\tau_{m_k}=\min_{n2^{-k-1}<j\leq n2^{-k}} \tau_j,
$$
Then $m_k\in\mathbb N$ ($0\leq k< N$) with
\begin{equation}\label{App-2}
n2^{-k-1}<m_k\leq n2^{-k}, \   \  m_{k+1}<m_k
\end{equation}
and
\begin{equation}\label{App-3}
\tau_{m_k}\leq \tau_j, \   \ \textrm{for } \  n2^{-k-1}<j\leq n2^{-k}.
\end{equation}

Setting $m_{N+1}:=1$ and taking $k=m_0$, it follows from \eqref{lemma3-sec2-1} that
\begin{eqnarray*}
\begin{aligned}
\sigma_n&\leq\left(\frac{m_0}{n}\right)^p\sigma_{m_0}+\tau_{m_0}
=n^{-p}\left(m_0^p\sigma_{m_0}+\tau_{m_0}\right)\\
&=n^{-p}\left(\sum_{k=0}^N m_k^p\left(\sigma_{m_k}
-\left(\frac{m_{k+1}}{m_k}\right)^p\sigma_{m_{k+1}}\right)-m_{N+1}^p\sigma_{m_{N+1}}\right)+\tau_{m_0}\\
&=n^{-p}\left(\sum_{k=0}^N m_k^p\left(\sigma_{m_k}
-\left(\frac{m_{k+1}}{m_k}\right)^p\sigma_{m_{k+1}}\right)-\sigma_{1}\right)+\tau_{m_0}.
\end{aligned}
\end{eqnarray*}
Using the equality \eqref{lemma3-sec2-1} again, we obtain from the above equation that
\begin{eqnarray*}
\begin{aligned}
\sigma_n&\leq
 n^{-p}\left(\sum_{k=0}^N m_k^p\left(\left(\frac{m_{k+1}}{m_k}\right)^p\sigma_{m_{k+1}}+\tau_{m_{k+1}}
-\left(\frac{m_{k+1}}{m_k}\right)^p\sigma_{k+1}\right)-\sigma_{1}\right)+\tau_{m_0}\\
&=
 n^{-p}\left(\sum_{k=0}^N  m_k^p\tau_{m_{k+1}}
-\sigma_{1}\right)+\tau_{m_0}\\
&\leq n^{-p}\sum_{k=0}^N  m_k^p\tau_{m_{k+1}}
+\tau_{m_0}.
\end{aligned}
\end{eqnarray*}
From \eqref{App-2}, it follows that
\begin{eqnarray*}
\sigma_n\leq n^{-p}\sum_{k=0}^N  \left(\frac{n}{2^k}\right)^p\tau_{m_{k+1}}
+\tau_{m_0}
=\sum_{k=1}^{N+1}  2^{(-k+1)p}\tau_{m_{k}}
+\tau_{m_0}
\leq 2^p\sum_{k=0}^{N+1}  2^{-kp}\tau_{m_{k}},
\end{eqnarray*}
which follows from \eqref{App-3} that
\begin{eqnarray*}
\begin{aligned}
\sigma_n&\leq  2^p\sum_{k=0}^{N+1}  2^{-kp}\frac{1}{n(2^{-k}-2^{-k-1})}\sum_{n2^{-k-1}<j\leq n2^{-k}}\tau_{j}\\
&=2^p\sum_{k=0}^{N+1}  \frac{2^{-kp}2^{k+1}}{n(n2^{-k-1})^{p-1}}\sum_{n2^{-k-1}<j\leq n2^{-k}}(n2^{-k-1})^{p-1}\tau_{j}\\
&=4^pn^{-p}\sum_{k=0}^{N+1}  \sum_{n2^{-k-1}<j\leq n2^{-k}}(n2^{-k-1})^{p-1}\tau_{j}.
\end{aligned}
\end{eqnarray*}
Clearly, the right side of the above equation is not greater than
\begin{eqnarray*}
\begin{aligned}
 4^pn^{-p}\sum_{k=0}^{N+1}  \sum_{n2^{-k-1}<j\leq n2^{-k}}j^{p-1}\tau_{j}
=& 4^pn^{-p}\left( \sum_{\frac n2<j\leq n}+\sum_{\frac n4<j\leq \frac n2}+\cdots+\sum_{\frac n{2^{N+1}}<j\leq \frac n{2^N}}\right)j^{p-1}\tau_{j}\\
=&4^pn^{-p}\sum_{\frac n{2^{N+1}}<j\leq n}j^{p-1}\tau_{j}\\
\leq& 4^pn^{-p}\sum_{j=1}^nj^{p-1}\tau_{j}.
\end{aligned}
\end{eqnarray*}
The proof of Lemma \ref{lemma3-sec2} is completed.
\end{proof}

\begin{proof}[Proof of Lemma \ref{lemma4-sec2}]
Set $\sigma_n=\nu_n$, $\tau_n=\psi_n$, and $p=s$. For \eqref{lemma4-sec2-2}, we use Lemma \ref{lemma3-sec2} and obtain
$$
\nu_n\leq 4^sn^{-s}\sum_{k=1}^nk^{s-1}\psi_k.
$$
From \eqref{lemma4-sec2-1} it follows that
$$
\mu_n\leq \left(\frac kn\right)^r\mu_k+4^sn^{-s}\sum_{j=1}^k j^{s-1}\psi_j+\psi_k
\leq \left(\frac kn\right)^r\mu_k+4^sn^{-s}\sum_{j=1}^n j^{s-1}\psi_j.
$$
Let $p=r$, $\sigma_n=\mu_n$, and $\tau_n=4^sn^{-s}\sum_{j=1}^nj^{s-1}\psi_j$ in the inequality \eqref{lemma3-sec2-1}, then we apply Lemma \ref{lemma3-sec2} again and obtain that
\begin{eqnarray*}
\begin{aligned}
\mu_n&\leq 4^rn^{-r}\sum_{k=1}^n k^{r-1}\psi_k+4^r4^sn^{-r}\sum_{k=1}^nk^{r-1-s}
\sum^k_{j=1}j^{s-1}\psi_j\\
&\leq 4^rn^{-r}\sum_{k=1}^n k^{r-1}\psi_k+4^{r+s}n^{-r}\sum^n_{j=1}j^{s-1}\psi_j\sum_{k=j}^nk^{r-1-s}.
\end{aligned}
\end{eqnarray*}
Using the fact that $\sum_{k=j}^{\infty}k^{r-1-s}\leq C_{rs}j^{r-s}$, the proof of Lemma \ref{lemma4-sec2} is completed.
\end{proof}

\section*{Acknowledgments}
The work was supported by the National Natural Science Foundation of China (Grants:  62176244 and 11971178) and Zhejiang Provincial Natural Science Foundation of China (Grant: LZ20F030001).

\bibliographystyle{elsarticle-num}
\bibliography{deep-net}

\begin{thebibliography}{10}
\expandafter\ifx\csname url\endcsname\relax
  \def\url#1{\texttt{#1}}\fi
\expandafter\ifx\csname urlprefix\endcsname\relax\def\urlprefix{URL }\fi
\expandafter\ifx\csname href\endcsname\relax
  \def\href#1#2{#2} \def\path#1{#1}\fi

\bibitem{goodfellow2016deep}
I.~Goodfellow, Y.~Bengio, A.~Courville, Deep Learning, MIT Press, 2016.

\bibitem{elbrachter2021deep}
D.~Elbr{\"a}chter, D.~Perekrestenko, P.~Grohs, H.~B{\"o}lcskei, Deep neural network approximation theory, IEEE Transactions on Information Theory 67~(5) (2021) 2581--2623.

\bibitem{guo2019realizing}
Z.-C. Guo, L.~Shi, S.-B. Lin, Realizing data features by deep nets, IEEE Transactions on Neural Networks and Learning Systems 31~(10) (2019) 4036--4048.

\bibitem{schmidt2020nonparametric}
J.~Schmidt-Hieber, Nonparametric regression using deep neural networks with relu activation function, Annals of Statistics 48~(4) (2020) 1875--1897.

\bibitem{sun2020global}
R.~Sun, D.~Li, S.~Liang, T.~Ding, R.~Srikant, The global landscape of neural networks: An overview, IEEE Signal Processing Magazine 37~(5) (2020) 95--108.

\bibitem{allen2019convergence}
Z.~Allen-Zhu, Y.~Li, Z.~Song, A convergence theory for deep learning via over-parameterization, in: International Conference on Machine Learning, PMLR, 2019, pp. 242--252.

\bibitem{lin2021generalization}
S.-B. Lin, Y.~Wang, D.-X. Zhou, Generalization performance of empirical risk minimization on over-parameterized deep relu nets, IEEE Transations on Information Theory 71~(3) (20215) 1978--1993.

\bibitem{han2020depth}
Z.~Han, S.~Yu, S.-B. Lin, D.-X. Zhou, Depth selection for deep relu nets in feature extraction and generalization, IEEE Transactions on Pattern Analysis and Machine Intelligence 44~(4) (2022) 1853--1868.

\bibitem{devore1993constructive}
R.~A. DeVore, G.~G. Lorentz, Constructive Approximation, Vol. 303, Springer Science \& Business Media, 1993.

\bibitem{hangelbroek2018inverse}
T.~Hangelbroek, F.~Narcowich, C.~Rieger, J.~Ward, An inverse theorem for compact lipschitz regions in ${R}^d$ using localized kernel bases, Mathematics of Computation 87~(312) (2018) 1949--1989.

\bibitem{qian2022neural}
Y.~Qian, D.~Yu, Neural network interpolation operators activated by smooth ramp functions, Analysis and Applications 20~(04) (2022) 791--813.

\bibitem{lin2018generalization}
S.-B. Lin, Generalization and expressivity for deep nets, IEEE Transactions on Neural Networks and Learning Systems 30~(5) (2018) 1392--1406.

\bibitem{chui2020realization}
C.~K. Chui, S.-B. Lin, B.~Zhang, D.-X. Zhou, Realization of spatial sparseness by deep relu nets with massive data, IEEE Transactions on Neural Networks and Learning Systems (2020).

\bibitem{Liuapproximating2023}
X.~Liu, Approximating smooth and sparse functions by deep neural networks: optimal approximation rates and saturation, Journal of Complexity 79~(101783) (2023).

\bibitem{wang2024component}
D.~Wang, S.-B. Lin, D.~Meng, F.~Cao, Component-based sketching for deep relu nets, arXiv preprint arXiv:2409.14174 (2024).

\bibitem{liu2022construction}
X.~Liu, D.~Wang, S.-B. Lin, Construction of deep relu nets for spatially sparse learning, IEEE Transactions on Neural Networks and Learning Systems (2022).

\bibitem{qian2022rates02}
Y.~Qian, D.~Yu, Rates of approximation by neural network interpolation operators, Applied Mathematics and Computation 418 (2022) 126781.

\bibitem{konovalov2008approximation}
V.~N. Konovalov, D.~Leviatan, V.~Maiorov, Approximation by polynomials and ridge functions of classes of s-monotone radial functions, Journal of Approximation Theory 152~(1) (2008) 20--51.

\bibitem{konovalov2009approximation}
V.~N. Konovalov, D.~Leviatan, V.~Maiorov, Approximation of sobolev classes by polynomials and ridge functions, Journal of Approximation Theory 159~(1) (2009) 97--108.

\bibitem{chui2019deep}
C.~K. Chui, S.-B. Lin, D.-X. Zhou, Deep neural networks for rotation-invariance approximation and learning, Analysis and Applications 17~(05) (2019) 737--772.

\bibitem{chui1994neural}
C.~K. Chui, X.~Li, H.~N. Mhaskar, Neural networks for localized approximation, Mathematics of Computation 63~(208) (1994) 607--623.

\bibitem{petersen2018optimal}
P.~Petersen, F.~Voigtlaender, Optimal approximation of piecewise smooth functions using deep relu neural networks, Neural Networks 108 (2018) 296--330.

\bibitem{yarotsky2017error}
D.~Yarotsky, Error bounds for approximations with deep relu networks, Neural Networks 94 (2017) 103--114.

\bibitem{zhou2020universality}
D.-X. Zhou, Universality of deep convolutional neural networks, Applied and Computational Harmonic Analysis 48~(2) (2020) 787--794.

\bibitem{chui1996limitations}
C.~K. Chui, X.~Li, H.~N. Mhaskar, Limitations of the approximation capabilities of neural networks with one hidden layer, Advances in Computational Mathematics 5~(1) (1996) 233--243.

\bibitem{lin2017does}
H.~W. Lin, M.~Tegmark, D.~Rolnick, Why does deep and cheap learning work so well?, Journal of Statistical Physics 168~(6) (2017) 1223--1247.

\bibitem{pinkus1999approximation}
A.~Pinkus, Approximation theory of the mlp model in neural networks, Acta Numerica 8 (1999) 143--195.

\bibitem{Ingo2021Approximation}
G.~Ingo, R.~Mones, Approximation rates for neural networks with encodable weights in smoothness spaces, Neural Networks 134 (2021) 107--130.
\newblock \href {https://doi.org/10.1016/j.neunet.2020.11.010} {\path{doi:10.1016/j.neunet.2020.11.010}}.

\bibitem{nakada2020adaptive}
R.~Nakada, M.~Imaizumi, Adaptive approximation and generalization of deep neural network with intrinsic dimensionality., Journal of Machine Learning Research 21 (2020) 174--1.

\bibitem{han2023learning}
Z.~Han, B.~Liu, S.-B. Lin, D.-X. Zhou, Deep convolutional neural networks with zero-padding: Feature extraction and learning, arXiv preprint arXiv:2307.16203 (2023).

\bibitem{mhaskar2016deep}
H.~N. Mhaskar, T.~Poggio, Deep vs. shallow networks: An approximation theory perspective, Analysis and Applications 14~(06) (2016) 829--848.

\bibitem{adams2003sobolev}
R.~A. Adams, J.~J. Fournier, Sobolev Spaces, Elsevier, 2003.

\bibitem{imaizumi2019deep}
M.~Imaizumi, K.~Fukumizu, Deep neural networks learn non-smooth functions effectively, in: The 22nd International Conference on Artificial Intelligence and Statistics, PMLR, 2019, pp. 869--878.

\bibitem{saad2020deep}
O.~M. Saad, A.~G. Hafez, M.~S. Soliman, Deep learning approach for earthquake parameters classification in earthquake early warning system, IEEE Geoscience and Remote Sensing Letters 18~(7) (2020) 1293--1297.

\bibitem{zhang2018deep}
G.~Zhang, Z.~Wang, Y.~Chen, Deep learning for seismic lithology prediction, Geophysical Journal International 215~(2) (2018) 1368--1387.

\bibitem{temirchev2020deep}
P.~Temirchev, M.~Simonov, R.~Kostoev, E.~Burnaev, I.~Oseledets, A.~Akhmetov, A.~Margarit, A.~Sitnikov, D.~Koroteev, Deep neural networks predicting oil movement in a development unit, Journal of Petroleum Science and Engineering 184 (2020) 106513.

\bibitem{bengio2013representation}
Y.~Bengio, A.~Courville, P.~Vincent, Representation learning: A review and new perspectives, IEEE Transactions on Pattern Analysis and Machine Intelligence 35~(8) (2013) 1798--1828.

\bibitem{shaham2018provable}
U.~Shaham, A.~Cloninger, R.~R. Coifman, Provable approximation properties for deep neural networks, Applied and Computational Harmonic Analysis 44~(3) (2018) 537--557.

\bibitem{zeng2021admm}
J.~Zeng, S.-B. Lin, Y.~Yao, D.-X. Zhou, On admm in deep learning: Convergence and saturation-avoidance, Journal of Machine Learning Research 22~(199) (2021) 1--67.

\bibitem{Chen2009The}
Z.~Chen, F.~Cao, The approximation operators with sigmoidal function, Computers and Mathematics with Applications 58~(4) (2009) 758--765.

\bibitem{Lorentz1966Approximation}
G.~Lorentz, Approximation of Functions, 1st Edition, Holt, Rinehart and Winston, New York, 1966.

\bibitem{Xie1998Approximation}
T.~Xie, S.~P. Zhou, Approximation Theory of real functions (in Chinese), Hangzhou University Press, Hangzhou, 1998.

\bibitem{Zygmund2002Trigonometric}
A.~Zygmund, Trigonometric Series, Cambridge University Press, New York, 2002.

\bibitem{lin2011essential}
S.~Lin, F.~Cao, Z.~Xu, The essential rate of approximation for radial function manifold, Science China Mathematics 54~(9) (2011) 1985--1994.

\bibitem{Ditzian1987Moduli}
Z.~Ditzian, V.~Totik, Moduli of Smoothness, 1st Edition, New York, 1987.

\bibitem{van1986steckin}
E.~Van~Wickeren, Steckin-marchaud-type inequalities in connection with bernstein polynomials, Constructive Approximation 2~(1) (1986) 331--337.

\end{thebibliography}

\end{document}